\documentclass[11pt]{article}
\usepackage[margin=1.25in]{geometry}
\usepackage{fancyhdr}
\usepackage[colorlinks,citecolor=blue,urlcolor=blue,linkcolor=blue,bookmarks=false]{hyperref}
\usepackage{amsfonts,epsfig,graphicx}
\usepackage{afterpage}
\usepackage{amsmath,amssymb,amsthm,mathabx} 
\usepackage{fullpage}
\usepackage[T1]{fontenc} 
\usepackage{epsf} 
\usepackage{graphics} 
\usepackage{amsfonts,amsmath}
\usepackage[sort,numbers]{natbib} 
\usepackage{psfrag,xspace}
\usepackage{color,etoolbox,enumitem}

\usepackage{times}

\RequirePackage{tablefootnote}
\RequirePackage{graphicx}

\allowdisplaybreaks[4]

\usepackage{algorithm2e}
\SetKwRepeat{Do}{do}{while}%
\SetKwInput{KwInput}{Input}
\SetKwInput{KwOutput}{Output}
\SetKwInput{KwFunction}{Function}
\SetKwInput{KwParameters}{Parameters}

\usepackage{authblk}

\newtheorem{definition}{Definition}

\newtheorem{theorem}{Theorem}
\newtheorem{lemma}{Lemma}

\newtheorem{proposition}{Proposition}

\theoremstyle{remark}
\newtheorem{remark}{Remark}
\newtheorem{example}{Example}

\newcommand{\mat}{\mathbf }
\newcommand{\vct}{\boldsymbol}

\newcommand{\ud}{\mathrm d}

\newcommand{\kl}{\mathrm{KL}}

\newcommand{\argmin}{\mathrm{argmin}}

\newcommand{\tr}{\mathrm{tr}}

\def\mP{\mathbb P}

\def\op{\mathrm{op}}

\def\sb{\mathfrak{b}}
\def\sv{\mathfrak{s}}
\def\sL{\mathfrak{L}}

\def\u{\mathsf{U}}
\def\l{\mathsf{L}}

\renewcommand{\hat}{\widehat}
\renewcommand{\tilde}{\widetilde}

\renewcommand{\check}{\widecheck}

\begin{document}

\title{Optimization of Smooth Functions with Noisy Observations: Local Minimax Rates}
\author[1]{Yining Wang}
\author[2]{Sivaraman Balakrishnan}
\author[1]{Aarti Singh}
\affil[1]{Machine Learning Department, Carnegie Mellon University}
\affil[2]{Department of Statistics and Data Science, Carnegie Mellon University}

\maketitle

%
%
%
%


\begin{abstract}
We consider the problem of \emph{global optimization} of an unknown non-convex smooth function
with zeroth-order feedback.
In this setup, an 
algorithm is allowed to adaptively query the underlying function at different locations
and receives noisy evaluations of function values at the queried points 
(i.e. the algorithm has access to zeroth-order information).
Optimization performance is evaluated by the expected difference of function values at the estimated optimum and the true optimum.
In contrast to the classical optimization setup, first-order information like gradients are \emph{not} directly 
accessible to the optimization algorithm. We show that the classical minimax framework of analysis, which roughly
characterizes the worst-case query complexity of an optimization algorithm in this setting, leads to excessively pessimistic results.
We propose a \emph{local minimax} framework to study the fundamental difficulty of optimizing smooth functions with adaptive function evaluations, 
which provides a refined picture of the intrinsic difficulty of zeroth-order optimization.
We show that for functions with fast level set growth around the global minimum, carefully designed optimization algorithms can 
identify a near global minimizer with many fewer queries.
For the special case of strongly convex and smooth functions, our implied convergence rates match the ones developed for zeroth-order \emph{convex} optimization problems \citep{flaxman2005online,agarwal2010optimal}.
At the other end of the spectrum, for worst-case smooth functions no algorithm can converge faster than the minimax rate of estimating the entire unknown function in the $\ell_\infty$-norm.
We provide an intuitive and efficient algorithm that attains the derived upper error bounds. Finally, using the local minimax framework we are able to clearly 
dichotomize adaptive and non-adaptive algorithms by showing that non-adaptive algorithms, although optimal in a global minimax sense, 
do not attain the optimal local minimax rate.

\end{abstract}

\section{Introduction}

Global function optimization with stochastic (zeroth-order) query oracles is an important problem in optimization, machine learning and statistics.
To optimize an unknown bounded function $f:\mathcal X\mapsto\mathbb R$ defined on a known compact $d$-dimensional domain $\mathcal X\subseteq\mathbb R^d$,
 the data analyst makes $n$ \emph{active} queries $x_1,\ldots,x_n\in\mathcal X$ and observes
\begin{equation}
y_t = f(x_t) + w_t, \;\;\;\;\;\; w_t\overset{i.i.d.}{\sim} \mathcal N(0, 1),\footnote{The exact distribution of the independent noise variables $\varepsilon_t$ is not important, and our results can be generalized to sub-Gaussian noise variables as well.}\;\; t=1,\ldots,n.
\label{eq:model}
\end{equation}
The queries $x_1,\ldots,x_t$ are \emph{active} in the sense that the selection of $x_t$ can depend on the previous queries and their responses $x_1,y_1,\ldots,x_{t-1},y_{t-1}$.
After $n$ queries, an estimate $\hat x_n\in\mathcal X$ is produced that approximately minimizes the unknown function $f$.
{
Such ``active query'' models are relevant in a broad range of (noisy) global optimization applications, for instance in hyper-parameter tuning of machine learning algorithms \citep{rasmussen2006gaussian} and sequential design
in material synthesis experiments where the goal is to maximize strengths of the produced materials \citep{reeja2012microwave,nakamura2017design}.
We refer the readers to Section~\ref{subsec:active} for a rigorous formulation of the active query model and contrast it with the classical passive query model.
}

The error of the estimate $\hat x_n$ is measured by the difference of $f(\hat x_n)$ and the \emph{global minimum} of $f$:
\begin{equation}
\sL(\hat x_n;f) := f(\hat x_n) - f^* \;\;\;\;\;\;\text{where}\;\; f^* := \inf_{x\in\mathcal X}f(x).
\end{equation}
To simplify our presentation, throughout the paper we take the domain $\mathcal X$ to be the $d$-dimensional unit cube $[0,1]^d$,
while our results can be easily generalized to other compact domains satisfying minimal regularity conditions.

When $f$ belongs to a smoothness class, say the H\"{o}lder class with exponent $\alpha$, a straightforward global optimization method is to first sample $n$ points uniformly at random from $\mathcal X$
and then construct nonparametric estimates $\hat f_n$ of $f$ using nonparametric regression methods such as (high-order) kernel smoothing or local polynomial regression \citep{tsybakov2009introduction,fan1996local}.
Classical analysis shows that the sup-norm reconstruction error $\|\hat f_n-f\|_\infty = \sup_{x\in\mathcal X}|\hat f_n(x)-f(x)|$
can be upper bounded by $\tilde O_\mP(n^{-\alpha/(2\alpha+d)})$\footnote{In the $\tilde O(\cdot)$ or $\tilde O_\mP(\cdot)$ notation we drop poly-logarithmic dependency on $n$}.
This global reconstruction guarantee 
then implies an $\tilde O_\mP(n^{-\alpha/(2\alpha+d)})$ upper bound on $\sL(\hat x_n;f)$ by 
considering $\hat x_n\in\mathcal X$ such that $\hat f_n(\hat x_n) = \inf_{x\in\mathcal X}\hat f_n(x)$
(such an $\hat x_n$ exists because $\mathcal X$ is closed and bounded).
Formally, we have the following proposition (proved in the Appendix) that converts a global reconstruction guarantee into an upper bound on optimization error:
\begin{proposition}
Suppose $\hat f_n(\hat x_n) = \inf_{x\in\mathcal X}\hat f_n(x)$. Then $\sL(\hat x_n;f)\leq 2\|\hat f_n-f\|_\infty$.
\label{prop:reduction}
\end{proposition}
Typically, fundamental limits on the optimal optimization error are understood through the lens of \emph{minimax analysis} where the object of study is the (global) minimax risk: 
\begin{equation}
\inf_{\hat x_n}\sup_{f\in\mathcal F} \mathbb E_f \sL(\hat x_n,f),
\label{eq:minimax}
\end{equation}
where $\mathcal F$ is a certain smoothness function class such as the H\"{o}lder class.
Although optimization appears to be easier than global reconstruction, we show in this paper that the $n^{-\alpha/(2\alpha+d)}$ 
rate is \emph{not} improvable in the global minimax sense in Eq.~(\ref{eq:minimax}) over H\"{o}lder classes.
Such a surprising phenomenon was also noted in previous works \citep{bull2011convergence,scarlett2017lower,hazan2017hyperparameter} for related problems.
On the other hand, extensive empirical evidence suggests that non-uniform/active allocations of query points can significantly reduce optimization error in practical global optimization of smooth, non-convex functions \citep{rasmussen2006gaussian}.
This raises the interesting question of understanding, from a theoretical perspective, under what conditions/in what scenarios is global optimization of smooth functions \emph{easier}
than their reconstruction, and the power of \emph{active/feedback-driven} queries that play important roles in global optimization.

In this paper, we propose a theoretical framework that partially answers the above questions.
In contrast to classical \emph{global} minimax analysis of nonparametric estimation problems, we adopt a \emph{local analysis}
which characterizes the optimal convergence rate of optimization error when the underlying function $f$ is within the neighborhood of a ``reference'' function $f_0$.
(See Section~\ref{subsec:local} for the rigorous local minimax formulation considered in this paper.)
%
Our main results are to characterize the local convergence rates $R_n(f_0)$ for a wide range of reference functions $f_0\in\mathcal F$.
More specifically, our contributions can be summarized as follows:
\begin{enumerate}[leftmargin=*]
\item We design an iterative (active) algorithm whose optimization error $\sL(\hat x_n;f)$ converges at a rate of $R_n(f_0)$ 
depending on the reference function $f_0$.
When the level sets of $f_0$ satisfy certain regularity and polynomial growth conditions, 
the local rate $R_n(f_0)$ can be upper bounded by
$R_n(f_0)=\tilde O(n^{-\alpha/(2\alpha+d-\alpha\beta)})$,
where $\beta\in[0,d/\alpha]$ is a parameter depending on $f_0$ that characterizes the volume growth of the \emph{level sets} of the reference function $f_0$.
(See assumption (A2), Proposition \ref{prop:upper-polynomial-growth} and Theorem \ref{thm:upper} for details).
The rate matches the global minimax convergence $n^{-\alpha/(2\alpha+d)}$ for worst-case $f_0$ where $\beta=0$,
but has the potential of being much faster when $\beta>0$. We emphasize that our algorithm has no knowledge of the reference function $f_0$ and achieves this rate adaptively. 

\item We prove \emph{local} minimax lower bounds that match the $n^{-\alpha/(2\alpha+d-\alpha\beta)}$ upper bound, up to logarithmic factors in $n$.
More specifically, we show that \emph{even if $f_0$ is known}, no (active) algorithm can estimate $f$ in close neighborhoods of $f_0$ at a rate faster than $n^{-\alpha/(2\alpha+d-\alpha\beta)}$.
We further show that, if active queries are not available and queries $x_1,\ldots,x_n$ are i.i.d.~uniformly sampled from $\mathcal X$, then the $n^{-\alpha/(2\alpha+d)}$ global minimax rate also applies locally regardless of how large $\beta$ is. 
Thus, there is an explicit gap between local minimax rates of active and uniform query models when $\beta$ is large.

\item In the special case when $f$ is \emph{convex}, the global optimization problem is usually referred to as \emph{zeroth-order convex optimization} and this problem has been widely studied
\citep{nemirovski1983problem,flaxman2005online,agarwal2010optimal,jamieson2012query,agarwal2013stochastic,bubeck2016kernel}.
Our results imply that, when $f_0$ is \emph{strongly} convex and smooth, the local minimax rate $R_n(f_0)$ is on the order of $\tilde O(n^{-1/2})$,
which matches the convergence rates in \citep{agarwal2010optimal}.
Additionally, our negative results (Theorem \ref{thm:lower}) indicate that the $n^{-1/2}$ rate cannot be achieved if $f_0$ is merely convex,
which seems to contradict $n^{-1/2}$ results in \citep{agarwal2013stochastic,bubeck2016kernel} that do not require strong convexity of $f$.
However, it should be noted that mere convexity of $f_0$ does \emph{not} imply convexity of $f$ in a neighborhood of $f_0$ (e.g., $\|f-f_0\|_\infty\leq\varepsilon$).
Our results show significant differences in the intrinsic difficulty of zeroth-order optimization of convex and near-convex functions.


\end{enumerate}




\begin{figure}[t]
\centering
\includegraphics[width=7.5cm]{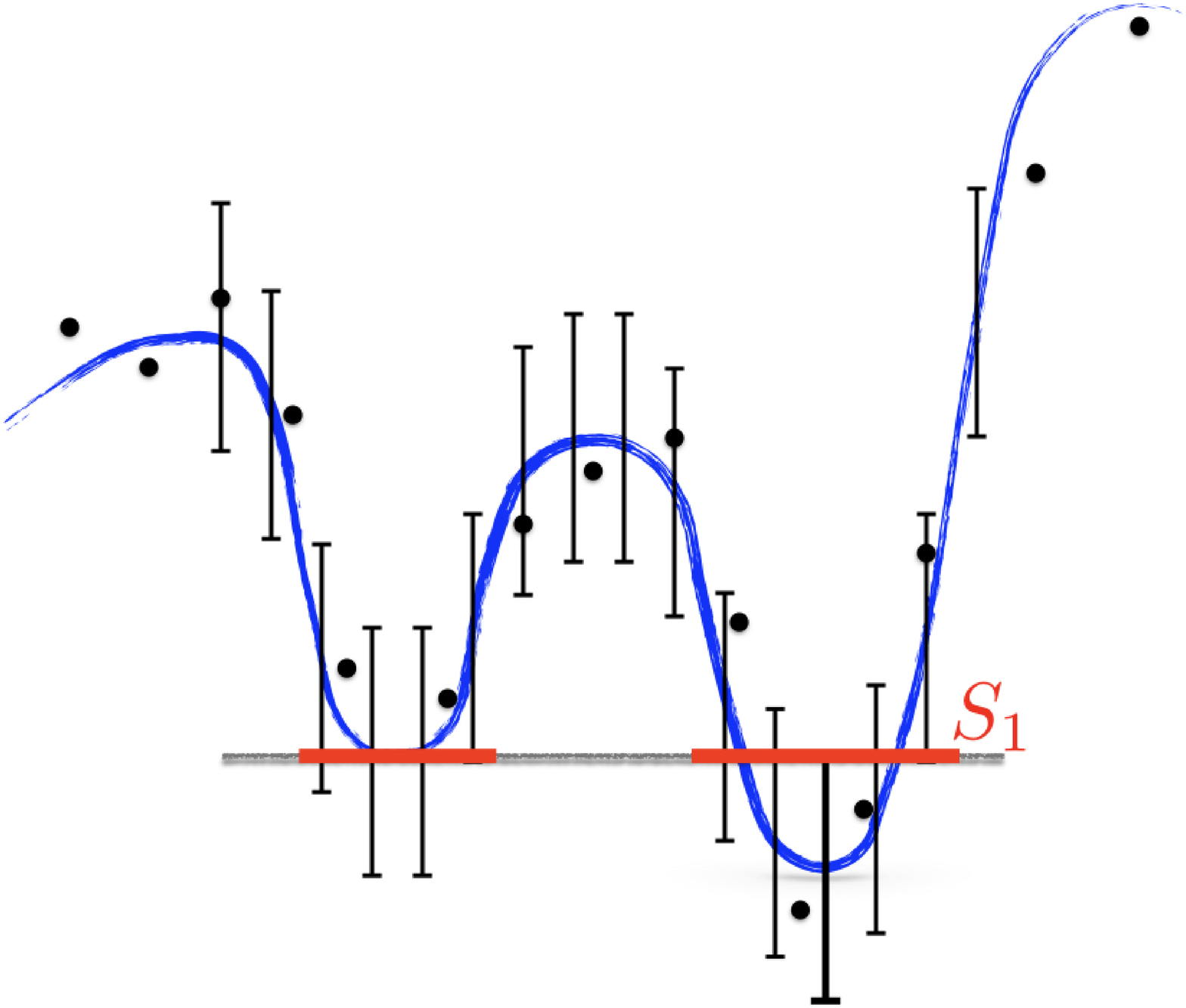}
\includegraphics[width=7.5cm]{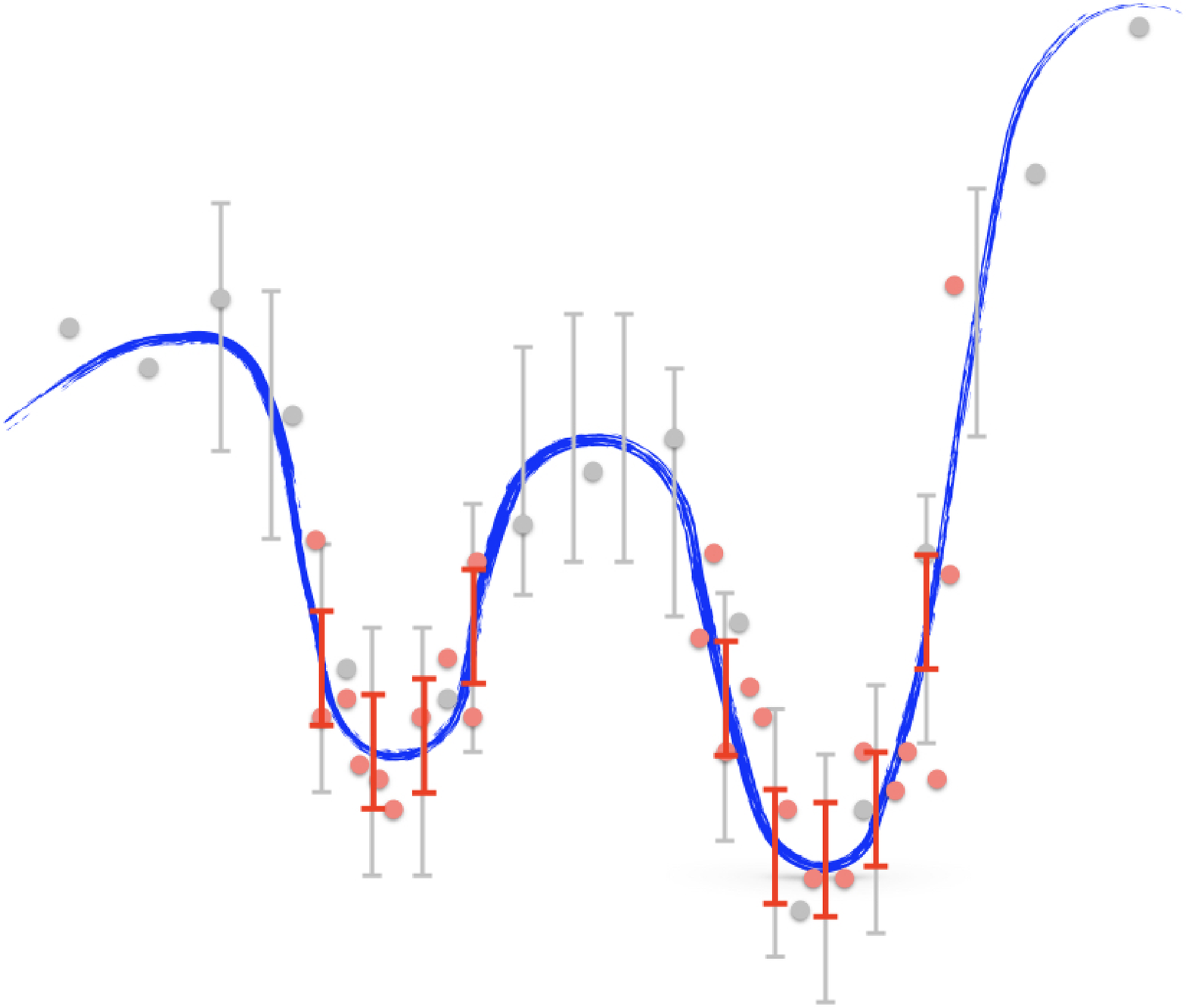}
\caption{\small\noindent Informal illustrations of Algorithm \ref{fig:main-alg}.
Solid blue curves depict the underlying function $f$ to be optimized, black and red solid dots denote the query points and their responses $\{(x_t,y_t)\}$,
and black/red vertical line segments correspond to uniform confidence intervals on function evaluations constructed using current batch of data observed.
The left figure illustrates the first epoch of our algorithm, where query points are uniformly sampled from the entire domain $\mathcal X$.
Afterwards, sub-optimal locations based on constructed confidence intervals are removed, and a shrinkt ``candidate set'' $S_1$ is obtained.
The algorithm then proceeds to the second epoch, illustrated in the right figure, where query points (in red) are sampled only from the restricted candidate set  
and shorter confidence intervals (also in red) are constructed and updated.
The procedure is repeated until $O(\log n)$ epochs are completed.
}
\label{fig:draws}
\end{figure}

%

\subsection{Related Work}

\emph{Global optimization},  known variously as \emph{black-box optimization}, \emph{Bayesian optimization} and the 
\emph{continuous-armed bandit},
has a long history in the optimization research community \citep{kan1987stochastic1,kan1987stochastic2}
and has also received a significant amount of recent interest in statistics and machine learning
\citep{hazan2017hyperparameter,rasmussen2006gaussian,bubeck2011x,malherbe2016ranking,malherbe2017global,bull2011convergence}.

Among the existing works, \citep{malherbe2016ranking,malherbe2017global} are perhaps the closest to our paper in terms of analytical perspectives.
Both papers impose additional assumptions on the level sets of the underlying function to obtain an improved convergence rate.
However, several important differences exist. 
First, the level set assumptions considered in the mentioned references 
are rather restrictive and essentially require the underlying function to be \emph{uni-modal},
while our assumptions are much more flexible and apply to multi-modal functions as well.
In addition, \cite{malherbe2016ranking,malherbe2017global} considered a \emph{noiseless} setting in which exact function evaluations $f(x_t)$ can be obtained,
while our paper studies the noise corrupted model in Eq.~(\ref{eq:model}) for which vastly different convergence rates are derived.
Finally, no matching lower bounds were proved in \citep{malherbe2016ranking,malherbe2017global}. 

The (stochastic) global optimization problem is similar to \emph{mode estimation} of either densities or regression functions,
which has a rich literature \citep{kiefer1952stochastic,purzen1962estimation,chen1988lower}.
An important difference between statistical mode estimation and global optimization is the way sample/query points $x_1,\ldots,x_n\in\mathcal X$ are distributed:
in mode estimation it is customary to assume the samples are independently and identically distributed,
while in global optimization sequential designs of samples/queries are allowed.
Furthermore, to estimate/locate the mode of an unknown density or regression function, such a mode has to be well-defined;
on the other hand, producing an estimate $\hat x_n$ with small $\sL(\hat x_n,f)$ is easier and results in weaker conditions imposed on the underlying function.

Methodology-wise, our proposed algorithm is conceptually similar to the abstract \emph{Pure Adaptive Search (PAS)} framework proposed and analyzed in \citep{zabinsky1992pure}.
The iterative procedure also resembles disagreement-based active learning methods \citep{balcan2009agnostic,dasgupta2008general,hanneke2007bound}
and the ``successive rejection'' algorithm in bandit problems \citep{even2006action}.
The intermediate steps of candidate point elimination can also be viewed as sequences of level set estimation problems \citep{polonik1995measuring,rigollet2009optimal,singh2009adaptive}
or cluster tree estimation \citep{chaudhuri2014consistent,balakrishnan2013cluster}
with active queries.


Another line of research has focused on \emph{first-order} optimization of quasi-convex or non-convex functions \citep{nesterov2006cubic,hazan2015beyond,ge2015escaping,agarwal2017finding,carmon2017convex,zhang2017hitting},
in which exact or unbiased evaluations of function \emph{gradients} are available at query points $x\in\mathcal X$.
\cite{zhang2017hitting} considered a Cheeger's constant restriction on level sets which is similar to our level set regularity assumptions (A2 and A2').
\cite{duchi2016local,duchi2016asymptotics} studied local minimax rates of first-order optimization of convex functions.
First-order optimization differs significantly from our setting because unbiased gradient estimation is generally impossible in the model of Eq.~(\ref{eq:model}).
Furthermore, most works on (first-order) non-convex optimization focus on convergence to stationary points or local minima, while 
we consider convergence to global minima.



\section{Background and Notation}
%
%

We first review standard asymptotic notation that will be used throughout this paper. For two sequences $\{a_n\}_{n=1}^{\infty}$ and $\{b_n\}_{n=1}^{\infty}$,
we write $a_n=O(b_n)$ or $a_n\lesssim b_n$ if $\limsup_{n\to\infty} |a_n|/|b_n| <\infty$, or equivalently $b_n=\Omega(a_n)$ or $b_n\gtrsim a_n$.
Denote $a_n=\Theta(b_n)$ or $a_n\asymp b_n$ if both $a_n\lesssim b_n$ and $a_n\gtrsim b_n$ hold.
We also write $a_n=o(b_n)$ or equivalently $b_n=\omega(a_n)$ if $\lim_{n\to\infty}|a_n|/|b_n|=0$.
For two sequences of random variables $\{A_n\}_{n=1}^\infty$ and $\{B_n\}_{n=1}^\infty$, denote $A_n=O_\mP(B_n)$
if for every $\epsilon>0$, there exists $C>0$ such that $\limsup_{n\to\infty}\Pr[|A_n|> C|B_n|]\leq\epsilon$. 
For $r>0$, $1\leq p\leq\infty$ and $x\in\mathbb R^d$, we denote $B_r^p(x) := \{z\in\mathbb R^d: \|z-x\|_p\leq r\}$ as the $d$-dimensional $\ell_p$-ball of radius $r$ centered at $x$,
where the vector $\ell_p$ norm is defined as $\|x\|_p := (\sum_{j=1}^d{|x_j|^p})^{1/p}$ for $1\leq p<\infty$ and $\|x\|_\infty := \max_{1\leq j\leq d}|x_j|$.
For any subset $S\subseteq\mathbb R^d$ we denote by $B_r^p(x;S)$ the set $B_r^p(x)\cap S.$


\subsection{Passive and Active Query Models}\label{subsec:active}

Let $U$ be a known random quantity defined on a probability space $\mathcal U$.
The following definitions characterize all passive and active optimization algorithms:
\begin{definition}[The passive query model]
Let $x_1,\ldots,x_n$ be i.i.d.~points uniformly sampled on $\mathcal X$ and $y_1,\ldots,y_n$
be observations from the model Eq.~(\ref{eq:model}).
A \emph{passive optimization algorithm} $\mathcal A$ with $n$ queries is parameterized by a mapping $\phi_n:(x_1,y_1,\ldots,x_n,y_n,U)\mapsto\hat x_n$
that maps the i.i.d.~observations $\{(x_i,y_i)\}_{i=1}^n$ to an estimated optimum $\hat x_n\in\mathcal X$, potentially randomized by $U$.
\label{defn:passive}
\end{definition}

\begin{definition}[The active query model]
An \emph{active optimization algorithm} can be parameterized by mappings $(\chi_1,\ldots,\chi_n,\phi_n)$, 
where for $t=1,\ldots,n$, 
\begin{align*}
\chi_t:(x_1,y_1,\ldots,x_{t-1},y_{t-1},U)\mapsto x_t
\end{align*}
produces a query point $x_t\in\mathcal X$ based on previous observations $\{(x_i,t_i)\}_{i=1}^{t-1}$,
and 
\begin{align*}
\phi_n:(x_1,y_1,\ldots,x_n,y_n,U)\mapsto \hat x_n
\end{align*} produces the final estimate.
All mappings $(\chi_1,\ldots,\chi_n,\phi_n)$ can be randomized by $U$.
\label{defn:active}
\end{definition}


\subsection{Local Minimax Rates}\label{subsec:local}

We use the classical \emph{local minimax analysis} \citep{van1998asymptotic} to understand the fundamental information-theoretical limits of noisy global optimization
of smooth functions.
On the upper bound side, we seek (active) estimators $\hat x_n$ such that
\begin{equation}
\sup_{f_0\in\Theta}\sup_{f\in\Theta', \|f-f_0\|_\infty \leq \varepsilon_n(f_0)} \Pr_f\left[\sL(\hat x_n;f)\geq C_1\cdot R_n(f_0)\right] \leq 1/4,
\label{eq:locally-minimax-ub}
\end{equation}
where $C_1>0$ is a positive constant.
Here $f_0\in\Theta$ is referred to as the \emph{reference function}, and $f\in\Theta'$ is the true underlying function which is assumed to be ``near'' $f_0$.
The minimax convergence rate of $\sL(\hat x_n;f)$ is then characterized \emph{locally} by $R_n(f_0)$ which depends on the reference function $f_0$. The constant of $1/4$ is chosen arbitrarily and any small constant leads to similar conclusions.
To establish negative results (i.e., locally minimax lower bounds), in contrast to the upper bound formulation,
we assume the potential active optimization estimator $\hat x_n$ has \emph{perfect knowledge} about the reference function $f_0\in\Theta$.
We then prove locally minimax lower bounds of the form
\begin{equation}
\inf_{\hat x_n}\sup_{f\in\Theta',\|f-f_0\|_\infty \leq \varepsilon_n(f_0)}
\Pr_f\left[\sL(\hat x_n;f)\geq C_2\cdot R_n(f_0)\right] \geq 1/3,
\label{eq:locally-minimax-lb}
\end{equation}
where $C_2>0$ is another positive constant and $\varepsilon_n(f_0), R_n(f_0)$ are desired local convergence rates for functions near the reference $f_0$.

Although in some sense classical, the local minimax definition we propose warrants further discussion.
\begin{enumerate}[leftmargin=*]
\item {\bf Roles of $\Theta$ and $\Theta'$: } The reference function $f_0$ and the true functions $f$ are assumed to belong to different but closely related function classes $\Theta$ and $\Theta'$.
In particular, in our paper $\Theta\subseteq\Theta'$, meaning that less restrictive assumptions are imposed on the true underlying function $f$ 
compared to those imposed on the reference function $f_0$ on which $R_n$ and $\varepsilon_n$ are based.

\item {\bf Upper Bounds: } It is worth emphasizing that the estimator $\widehat{x}_n$ has no knowledge of the reference function $f_0$. From the perspective of upper bounds, we can consider the simpler task of producing $f_0$-dependent bounds (eliminating the second supremum) to 
instead study the (already interesting) quantity:
\begin{align*}
\sup_{f_0\in\Theta}\Pr_{f_0}\left[\sL(\hat x_n;f_0)\geq C_1R_n(f_0)\right] \leq 1/4.
\end{align*}
As indicated above we maintain the double-supremum in the definition because fewer assumptions are imposed directly on the true underlying function $f$, and further because it allows 
to more directly compare our upper and lower bounds.

\item {\bf Lower Bounds and the 
choice of the ``localization radius'' $\varepsilon_n(f_0)$: } Our lower bounds allow the estimator knowledge of the reference function (this makes establishing the lower bound more challenging). Eq.~(\ref{eq:locally-minimax-lb}) implies that no estimator $\hat x_n$ can effectively optimize a function $f$ close to $f_0$ beyond the convergence rate of $R_n(f_0)$, even if perfect knowledge of the reference function $f_0$ is available a priori. The $\varepsilon_n(f_0)$ parameter that decides the ``range'' in which local minimax rates apply is taken to be on the same order as the actual local rate $R_n(f_0)$
in this paper. This is (up to constants) the smallest radius for which we can hope to obtain non-trivial lower-bounds: if we consider a much smaller radius than $R_n(f_0)$ then the trivial estimator which outputs the minimizer of the reference function would achieve a faster rate than $R_n(f_0)$. Selecting the smallest possible radius makes establishing the lower bound most challenging but provides a refined picture of the complexity of zeroth-order optimization.

%

\end{enumerate}

\section{Main Results}
With this background in place we now turn our attention to our main results. We begin by collecting our assumptions about the true underlying function and the reference function in Section~\ref{sec:ass}. We state and discuss the consequences of our upper and lower bounds in~Sections~\ref{subsec:upper} and~\ref{subsec:lower} respectively. We defer most technical proofs
to the Appendix and turn our attention to our optimization algorithm in~Section~\ref{sec:algorithm}.

\subsection{Assumptions}
\label{sec:ass}

We first state and motivate assumptions that will be used. 
The first assumption states that $f$ is locally H\"{o}lder smooth on its level sets.
\begin{enumerate}[leftmargin=0.5in]
\item[(A1)] 
There exist constants $\kappa,\alpha,M>0$ such that $f$ restricted on $\mathcal X_{f,\kappa}:=\{x\in\mathcal X: f(x)\leq f^*+\kappa\}$
belongs to the H\"{o}lder class $\Sigma^\alpha(M)$, meaning that $f$ is $k$-times differentiable on $\mathcal X_{f,\kappa}$ and furthermore
for any $x,x'\in\mathcal X_{f,\kappa}$,
\footnote{the particular $\ell_\infty$ norm is used for convenience only and can be replaced by any equivalent vector norms.}
\begin{equation}
\sum_{j=0}^k\sum_{\alpha_1+\ldots+\alpha_d=j}|f^{(\vct\alpha,j)}(x)| + \sum_{\alpha_1+\ldots+\alpha_d=k}\frac{|f^{(\vct\alpha,k)}(x)-f^{(\vct\alpha,k)}(x')|}{\|x-x'\|_\infty^{\alpha-k}}\leq M. 
\label{eq:holder}
\end{equation}
Here $k=\lfloor \alpha\rfloor$ is the largest integer lower bounding $\alpha$ and $f^{(\vct\alpha,j)}(x) := \partial^j f(x)/\partial x_1^{\alpha_1}\ldots\partial x_d^{\alpha_d}$.
\end{enumerate}

{
We use $\Sigma_\kappa^\alpha(M)$ to denote the class of all functions satisfying (A1).
We remark that (A1) is weaker than the standard assumption that $f$ on its entire domain $\mathcal X$ belongs to the H\"{o}lder class $\Sigma^\alpha(M)$.
This is because places with function values larger than $f^*+\kappa$ can be easily detected and removed by a pre-processing step.
We give further details of the pre-processing step in Section~\ref{subsec:relax}.
}


Our next assumption concern the ``regularity'' of the \emph{level sets} of the ``reference'' function $f_0$. 
Define $L_{f_0}(\epsilon) := \{x\in\mathcal X: f_0(x)\leq f_0^*+\epsilon\}$ as the $\epsilon$-level set of $f_0$,
and $\mu_{f_0}(\epsilon) := \lambda(L_{f_0}(\epsilon))$ as the Lebesgue measure of $L_{f_0}(\epsilon)$,
also known as the \emph{distribution function}.
Define also $N(L_{f_0}(\epsilon),\delta)$ as the smallest number of $\ell_2$-balls of radius $\delta$ that cover $ L_{f_0}(\epsilon)$.

\begin{enumerate}[leftmargin=0.5in]
\item[(A2)]There exist constants $c_0>0$ and $C_0>0$ such that $N(L_{f_0}(\epsilon),\delta) \leq C_0[1 + \mu_{f_0}(\epsilon)\delta^{-d}]$ 
for all $\epsilon,\delta\in(0,c_0]$.
\end{enumerate}
We use $\Theta_{\mat C}$ to denote all functions that satisfy (A2) with respect to parameters $\mat C=(c_0,C_0)$.


At a higher level, the regularity condition (A2) assumes that the level sets are sufficiently ``regular'' such that covering them with small-radius balls does not require significantly larger total volumes.
For example, consider a perfectly regular case of $L_{f_0}(\epsilon)$ being the $d$-dimensional $\ell_2$ ball of radius $r$:
$L_{f_0}(\epsilon) = \{x\in\mathcal X: \|x-x^*\|_2\leq r\}$.
Clearly, $\mu_{f_0}(\epsilon) \asymp r^d$.
In addition, the $\delta$-covering number in $\ell_2$ of $L_{f_0}(\epsilon)$ is on the order of $1+(r/\delta)^d \asymp 1 + \mu_{f_0}(\epsilon)\delta^{-d}$, which satisfies the scaling in (A2).

{When (A2) holds, uniform confidence intervals of $f$ on its level sets are easy to construct because little statistical efficiency is lost by slightly enlarging the level sets 
so that complete $d$-dimensional cubes are contained in the enlarged level sets.
On the other hand, when regularity of level sets fails to hold such nonparametric estimation can be very difficult or even impossible.
As an extreme example, suppose the level set $L_{f_0}(\epsilon)$ consists of $\mathfrak n$ standalone and well-spaced points in $\mathcal X$: the Lebesgue measure of $L_{f_0}(\epsilon)$ would be zero,
but at least $\Omega(\mathfrak n)$ queries are necessary to construct uniform confidence intervals on $L_{f_0}(\epsilon)$.
It is clear that such $L_{f_0}(\epsilon)$ violates (A2), because $N(L_{f_0}(\epsilon),\delta)\geq \mathfrak n$ as $\delta\to 0^+$ but $\mu_{f_0}(\epsilon) = 0$.
}

\subsection{Upper Bound}\label{subsec:upper}


%

The following theorem is our main result that upper bounds the local minimax rate of noisy global optimization with active queries.
\begin{theorem} 
For any $\alpha,M,\kappa,c_0,C_0>0$ and $f_0\in\Sigma_\kappa^\alpha(M)\cap\Theta_{\mat C}$, where $\mat C=(c_0,C_0)$, define
\begin{equation}
\varepsilon_n^\u(f_0) := \sup\left\{\varepsilon>0: \varepsilon^{-(2+d/\alpha)}\mu_{f_0}(\varepsilon) \geq n/\log^\omega n\right\},
\end{equation}
where $\omega > 5+d/\alpha$ is a large constant.
Suppose also that $\varepsilon_n^\u(f_0)\to 0$ as $n\to\infty$.
Then for sufficiently large $n$, there exists an estimator $\hat x_n$ with access to $n$ active queries $x_1,\ldots,x_n\in\mathcal X$,
 a constant $C_R>0$ depending only on $\alpha,M,\kappa,c,c_0,C_0$ and
 a constant $\gamma>0$ depending only on $\alpha$ and $d$ such that
\begin{equation}
\sup_{f_0\in\Sigma_\kappa^\alpha(M)\cap\Theta_{\mat C}}\sup_{\substack{f\in\Sigma_\kappa^\alpha(M),\\ \|f-f_0\|_\infty\leq \varepsilon_n^\u(f_0)}} \Pr_f\left[\sL(\hat x_n,f) > C_R\log^\gamma n \cdot (\varepsilon^\u_n(f_0)+n^{-1/2})\right] \leq 1/4.
\end{equation}
\label{thm:upper}
\end{theorem}

\begin{remark}
Unlike the (local) smoothness class $\Sigma_\kappa^\alpha(M)$, the additional function class $\Theta_{\mat C}$ that encapsulates (A2) is imposed only on the ``reference'' function $f_0$ but not the true function $f$ to be estimated. 
This makes the assumptions considerably weaker because the true function $f$ may violate either or both (A2) while our results remain valid.
\label{rem:f0}
\end{remark}

\begin{remark}
The estimator $\hat x_n$ does \emph{not} require knowledge of parameters $\kappa,c_0,C_0$ or $\varepsilon_n^\u(f_0)$, and automatically adapts to them, as shown in the next section.
It however requires knowledge of $\alpha$ and $M$, parameters of the smooth function class.
Such knowledge is unlikely to be optional, as the key step of building 
honest confidence intervals that adapt to $\alpha$ and/or $M$ is very difficult and in general not possible without additional assumptions, as shown by \cite{low1997nonparametric,cai2004adaptation}.
\label{rem:adaptivity}
\end{remark}

{
\begin{remark}
When the distribution function $\mu_{f_0}(\epsilon)$ does not change abruptly with $\epsilon$ the expression of $\varepsilon_n^\u(f_0)$ can be significantly simplified.
In particular, if for all $\epsilon\in(0,c_0]$ it holds that
\begin{equation}
\mu_{f_0}(\epsilon/\log n) \geq \mu_{f_0}(\epsilon) / [\log n]^{O(1)},
\label{eq:stable-muf}
\end{equation}
then $\varepsilon_n^{\u}(f_0)$ can be upper bounded as
\begin{equation}
\varepsilon_n^\u(f_0) \leq [\log n]^{O(1)}\cdot \sup\left\{\varepsilon>0: \varepsilon^{-(2+d/\alpha)}\mu_{f_0}(\varepsilon)\geq n\right\}.
\label{eq:stable-muf-2}
\end{equation}
It is also noted that if $\mu_{f_0}(\epsilon)$ has a polynomial behavior of $\mu_{f_0}(\epsilon)\asymp \epsilon^\beta$ for some constant $\beta\geq 0$,
then Eq.~(\ref{eq:stable-muf}) is satisfied and so is Eq.~(\ref{eq:stable-muf-2}).
\label{rem:stable-muf}
\end{remark}
}

{
The quantity $\varepsilon_n^\u(f_0)=\inf\{\varepsilon>0: \varepsilon^{-(2+d/\alpha)}\mu_{f_0}(\varepsilon)\geq n/\log^\omega n\}$ is crucial in determining the convergence rate of optimization error of $\hat x_n$
\emph{locally} around the reference function $f_0$.
While the definition of $\varepsilon_n^\u(f_0)$ is mostly implicit and involves solving an inequality concerning the distribution function $\mu_{f_0}(\cdot)$,
we remark that it admits a simple form when $\mu_{f_0}$ has a polynomial growth rate, as shown by the following proposition:
\begin{proposition}
Suppose $\mu_{f_0}(\epsilon) \lesssim \epsilon^\beta$ for some constant $\beta\in[0,2+d/\alpha)$.
Then $\varepsilon_n^\u(f_0) = \tilde O(n^{-\alpha/(2\alpha+d-\alpha\beta)})$.
In addition, if $\beta\in[0,d/\alpha]$ then $\varepsilon_n^\u(f_0)+n^{-1/2} \lesssim \varepsilon_n^\u(f_0) = \tilde O(n^{-\alpha/(2\alpha+d-\alpha\beta)})$.
\label{prop:upper-polynomial-growth}
\end{proposition}

Proposition \ref{prop:upper-polynomial-growth} can be easily verified by solving the system $\varepsilon^{-(2+d/\alpha)}\mu_{f_0}(\varepsilon)\geq n/\log^\omega n$
with the condition $\mu_{f_0}(\epsilon)\lesssim \epsilon^\beta$.
We therefore omit its proof.
The following two examples give some simple reference functions $f_0$ that satisfy the $\mu_{f_0}(\epsilon)\lesssim \epsilon^\beta$ condition in Proposition \ref{prop:upper-polynomial-growth}
with particular values of $\beta$.

\begin{example}
The constant function $f_0\equiv 0$ satisfies (A1) through (A3) with $\beta=0$.
\label{exmp:const}
\end{example}

\begin{example}
$f_0\in\Sigma_\kappa^2(M)$ that is \emph{strongly convex}
\footnote{A twice differentiable function $f_0$ is strongly convex if there exists $\sigma>0$ such that $\nabla^2 f_0(x)\succeq \sigma I, \forall x\in\mathcal X$.}
 satisfies (A1) through (A3) with $\beta=d/2$.
 \label{exmp:convex}
\end{example}

{
Example \ref{exmp:const} is simple to verify, as the volume of level sets of the constant function $f_0\equiv 0$ exhibits a phase transition at $\epsilon=0$ and $\epsilon>0$,
rendering $\beta=0$ the only parameter option for which $\mu_{f_0}(\epsilon)\lesssim \epsilon^\beta$.
Example \ref{exmp:convex} is more involved, and holds because the strong convexity of $f_0$ \emph{lower bounds} the growth rate of $f_0$ when moving away from its minimum.
We give a rigorous proof of Example \ref{exmp:convex} in the appendix.
We also remark that $f_0$ does \emph{not} need to be exactly strongly convex for $\beta=d/2$ to hold, and the example is valid for, e.g., piecewise strongly convex functions with 
a constant number of pieces too.
}
}

{
To best interpret the results in Theorem \ref{thm:upper} and Proposition \ref{prop:upper-polynomial-growth}, it is instructive to compare the ``local'' rate $n^{-\alpha/(2\alpha+d-\alpha\beta)}$ with the baseline rate $n^{-\alpha/(2\alpha+d)}$,
which can be attained by reconstructing $f$ in sup-norm and applying Proposition \ref{prop:reduction}.
Since $\beta\geq 0$, the local convergence rate established in Theorem \ref{thm:upper} is never slower,
and the improvement compared to the baseline rate $n^{-\alpha/(2\alpha+d)}$ is dictated by $\beta$, which governs the growth rate of volume of level sets of the reference function $f_0$.
In particular, for functions that grows fast when moving away from its minimum, the parameter $\beta$ is large and therefore the local convergence rate around $f_0$ could be much faster than $n^{-\alpha/(2\alpha+d)}$.

Theorem \ref{thm:upper} also implies concrete convergence rates for special functions considered in Examples \ref{exmp:const} and \ref{exmp:convex}.
For the constant reference function $f_0\equiv 0$, Example \ref{exmp:const} and Theorem \ref{thm:upper} yield that $R_n(f_0) \asymp n^{-\alpha/(2\alpha+d)}$, 
which matches the baseline rate $n^{-\alpha/(2\alpha+d)}$ and suggests that $f_0\equiv 0$ is the worst-case reference function.
This is intuitive, because $f_0\equiv 0$ has the most drastic level set change at $\epsilon\to 0^+$ and therefore small perturbations anywhere of $f_0$ result in changes of the optimal locations.
On the other hand, if $f_0$ is strongly smooth and convex as in Example \ref{exmp:convex}, Theorem \ref{thm:upper} suggests that $R_n(f_0) \asymp n^{-1/2}$, 
which is significantly better than the $n^{-2/(4+d)}$ baseline rate \footnote{Note that $f_0$ being strongly smooth implies $\alpha=2$ in the local smoothness assumption.}
and also matches existing works on zeroth-order optimization of convex functions \citep{agarwal2010optimal}.
The faster rate holds intuitively because strongly convex functions grows fast when moving away from the minimum, which implies small level set changes.
An active query algorithm could then focus most of its queries onto the small level sets of the underlying function, resulting in more accurate local function reconstructions and faster optimization error rate.

Our proof of Theorem \ref{thm:upper} is constructive, by upper bounding the local minimax optimization error of an explicit algorithm. 
At a higher level, the algorithm partitions the $n$ active queries evenly into $\log n$ epochs, and level sets of $f$ are estimated at the end of each epoch
by comparing (uniform) confidence intervals on a dense grid on $\mathcal X$.
It is then proved that the volume of the estimated level sets contracts \emph{geometrically}, until the target convergence rate $R_n(f_0)$ is attained.
The complete proof of Theorem \ref{thm:upper} is placed in Section~\ref{subsec:proof-upper}.
}

\subsection{Lower Bounds}
\label{subsec:lower}

{
We prove local minimax lower bounds that match the upper bounds in Theorem \ref{thm:upper} up to logarithmic terms.
As we remarked in Section~\ref{subsec:local}, in the local minimax lower bound formulation we assume the data analyst has full knowledge of the reference function $f_0$,
which makes the lower bounds stronger as more information is available a priori.

To facilitate such a strong local minimax lower bounds, the following additional condition is imposed on the reference function $f_0$
of which the data analyst has perfect information.
}

\begin{enumerate}[leftmargin=0.5in]
\item[(A2')] 
There exist constants $c_0',C_0'>0$ such that $M(L_{f_0}(\epsilon),\delta) \geq C_0'\mu_{f_0}(\epsilon)\delta^{-d}$ for all $\epsilon,\delta\in(0,c_0']$,
where $M(L_{f_0}(\epsilon),\delta)$ is the maximum number of disjoint $\ell_2$ balls
of radius $\delta$ that can be packed into $L_{f_0}(\epsilon)$.
\end{enumerate}
We denote $\Theta'_{\mat C'}$ as the class of functions that satisfy (A2') with respect to parameters $\mat C'=(c_0',C_0')>0$.
Intuitively, (A2') can be regarded as the ``reverse'' version of (A2), which basically means that (A2) is ``tight''.
%

We are now ready to state our main negative result, which shows, from an information-theoretical perspective, that the upper bound in Theorem \ref{thm:upper} is not improvable.
\begin{theorem}
Suppose $\alpha,c_0,C_0,c_0',C_0'>0$ and $\kappa=\infty$.
Denote $\mat C=(c_0,C_0)$ and $\mat C'=(c_0',C_0')$.
For any $f_0\in\Theta_{\mat C}\cap\Theta'_{\mat C'}$, define
\begin{equation}
\varepsilon^\l_n(f_0) := \sup\left\{\varepsilon > 0: \varepsilon^{-(2+d/\alpha)}\mu_{f_0}(\varepsilon) \geq n\right\}.
\end{equation}
Then there exist constant $M>0$ depending on $\alpha,d,\mat C,\mat C'$ such that,
for any $f_0\in\Sigma_\kappa^\alpha(M/2)\cap\Theta_{\mat C}\cap\Theta_{\mat C'}$, 
\begin{equation}
\inf_{\hat x_n}\sup_{\substack{f\in\Sigma_\kappa^\alpha(M),\\\|f-f_0\|_\infty\leq 2\varepsilon_n^\l(f_0)}} 
\Pr_{f}\left[\sL(\hat x_n;f) \geq \varepsilon_n^\l(f_0)\right] \geq \frac{1}{3}.
\end{equation}
\label{thm:lower}
\end{theorem}
\begin{remark}
For any $f_0$ and $n$ it always holds that $\varepsilon_n^\l(f_0)\leq \varepsilon_n^\u(f_0)$.
\label{rem:eps-n-lower}
\end{remark}

{
\begin{remark}
If the distribution function $\mu_{f_0}(\epsilon)$ satisfies Eq.~(\ref{eq:stable-muf}) in Remark \ref{rem:stable-muf},
then $\varepsilon_n^\l(f_0) \geq \varepsilon_n^\u(f_0)/[\log n]^{O(1)}$.
\label{rem:eps-n-close}
\end{remark}
}
%

{
Remark \ref{rem:eps-n-lower} shows that there might be a gap between the locally minimax upper and lower bounds in Theorems \ref{thm:upper} and \ref{thm:lower}.
Nevertheless, Remark \ref{rem:eps-n-close} shows that under the mild condition of $\mu_{f_0}(\epsilon)$ does not change too abruptly with $\epsilon$,
the gap between $\varepsilon_n^\u(f_0)$ and $\varepsilon_n^\l(f_0)$ is only a poly-logarithmic term in $n$.
Additionally, the following proposition derives explicit expression of $\varepsilon_n^\l(f_0)$ for reference functions whose distribution functions have a polynomial growth,
which matches the Proposition \ref{prop:upper-polynomial-growth} up to $\log n$ factors.
Its proof is again straightforward.
\begin{proposition}
Suppose $\mu_{f_0}(\epsilon)\gtrsim \epsilon^\beta$ for some $\beta\in[0,2+d/\alpha)$.
Then $\varepsilon_n^\l(f_0) = \Omega(n^{-\alpha/(2\alpha+d-\alpha\beta)})$.
\label{prop:lower-polynomial-growth}
\end{proposition}

The following proposition additionally shows the existence of $f_0\in\Sigma_\infty^\alpha(M)\cap\Theta_{\mat C}\cap\Theta_{\mat C'}$ that satisfies $\mu_{f_0}(\epsilon)\asymp \epsilon^\beta$ for any values of 
$\alpha>0$ and $\beta\in[0,d/\alpha]$.
Its proof is given in the appendix.
\begin{proposition}
Fix arbitrary $\alpha,M>0$ and $\beta\in[0,d/\alpha]$.
There exists $f_0\in\Sigma_\kappa^\alpha(M)\cap\Theta_{\mat C}\cap\Theta_{\mat C'}$
for $\kappa=\infty$ and constants $\mat C=(c_0,C_0)$, $\mat C'=(c_0',C_0')$ that depend only on $\alpha,\beta,M$ and $d$
such that $\mu_{f_0}(\epsilon)\asymp \epsilon^\beta$.
\label{prop:negative-examples}
\end{proposition}

}

{
Theorem \ref{thm:lower} and Proposition \ref{prop:lower-polynomial-growth} show that the $n^{-\alpha/(2\alpha+d-\alpha\beta)}$ upper bound on local minimax convergence rate established in Theorem \ref{thm:upper}
is not improvable up to logarithmic factors of $n$.
Such information-theoretical lower bounds on the convergence rates hold \emph{even if the data analyst has perfect information of $f_0$}, the reference function on which the $n^{-\alpha/(2\alpha+d-\alpha\beta)}$ local
rate is based.
Our results also imply an $n^{-\alpha/(2\alpha+d)}$ minimax lower bound over all $\alpha$-H\"{o}lder smooth functions, showing that without additional assumptions, noisy optimization of smooth functions is as difficult as 
reconstructing the unknown function in sup-norm.

Our proof of Theorem \ref{thm:lower} also differs from existing minimax lower bound proofs for active nonparametric models \citep{castro2008minimax}.
The classical approach is to invoke Fano's inequality and to upper bound the KL divergence between different underlying functions $f$ and $g$ using $\|f-g\|_\infty$,
corresponding to the point $x\in\mathcal X$ that leads to the largest KL divergence.
Such an approach, however, does not produce tight lower bounds for our problem.
To overcome such difficulties, we borrow the lower bound analysis for bandit pure exploration problems in \citep{bubeck2009pure}.
In particular, our analysis considers the query distribution of any active query algorithm $\mathcal A=(\varphi_1,\ldots,\varphi_n,\phi_n)$ under the reference function $f_0$
and bounds the perturbation in query distributions between $f_0$ and $f$ using Le Cam's lemma.
Afterwards, an adversarial function choice $f$ can be made based on the query distributions of the considered algorithm $\mathcal A$.
We defer the complete proof of Theorem \ref{thm:lower} to Section~\ref{subsec:proof-lower}.
}

Theorem \ref{thm:lower} applies to any global optimization method that makes \emph{active} queries,
corresponding to the query model in Definition \ref{defn:active}.
The following theorem, on the other hand, shows that for passive algorithms (Definition \ref{defn:passive}) the $n^{-\alpha/(2\alpha+d)}$ optimization rate is not improvable even 
with additional level set assumptions imposed on $f_0$.
This demonstrates an explicit gap between passive and adaptive query models in global optimization problems.
\begin{theorem}
Suppose $\alpha,c_0,C_0,c_0',C_0'>0$ and $\kappa=\infty$.
Denote $\mat C=(c_0,C_0)$ and $\mat C'=(c_0',C_0')$.
Then there exist constant $M>0$ depending on $\alpha,d,\mat C,\mat C'$ 
and $N$ depending on $M$ 
such that, for any $f_0\in\Sigma_\kappa^\alpha(M/2)\cap\Theta_{\mat C}\cap\Theta_{\mat C'}$ satisfying $\varepsilon_n^\l(f_0)\leq\tilde\varepsilon_n^\l =: [\log n/n]^{\alpha/(2\alpha+d)}$,
\begin{equation}
\inf_{\check x_n}\sup_{\substack{f\in\Sigma_\kappa^\alpha(M),\\\|f-f_0\|_\infty\leq2 \tilde\varepsilon_n^\l}} 
\Pr_f\left[\sL(\hat x_n;f) \geq \tilde\varepsilon_n^\l\right] \geq \frac{1}{3}
\;\;\;\;\;\;\text{for all }n\geq N.
\label{eq:lower-passive}
\end{equation}
\label{thm:lower-passive}
\end{theorem}

{
Intuitively, the apparent gap demonstrated by Theorems \ref{thm:lower} and \ref{thm:lower-passive} between the active and passive query models stems from the observation that, 
a passive algorithm $\mathcal A$ only has access to uniformly sampled query points $x_1,\ldots,x_n$ and therefore cannot focus on a small level set of $f$ in order to improve query efficiency.
In addition, for functions that grow faster when moving away from their minima (implying a larger value of $\beta$), 
the gap between passive and active query models becomes bigger as active queries can more effectively exploit the restricted level sets of such functions.
}

\section{Our Algorithm}\label{sec:algorithm}

In this section we describe a concrete algorithm that attains the upper bound in Theorem \ref{thm:upper}.
We start with a cleaner algorithm that operates under the slightly stronger condition that $\kappa=\infty$ in (A1),
meaning that $f$ is $\alpha$-H\"{o}lder smooth on the entire domain $\mathcal X$.
The generalization to $\kappa>0$ being a constant is given in Section~\ref{subsec:relax} with an additional pre-processing step.

Let $G_n\in\mathcal X$ be a \emph{finite} grid of points in $\mathcal X$. 
We assume the finite grid $G_n$ satisfies the following two mild conditions: 
\begin{enumerate}
\item[(B1)] Points in $G_n$ are sampled i.i.d.~from an unknown distribution $P_X$ on $\mathcal X$;
furthermore, the density $p_X$ associated with $P_X$ satisfies $\underline p_0\leq p_X(x)\leq\overline p_0$ for all $x\in\mathcal X$,
where $0<\underline p_0\leq\overline p_0<\infty$ are uniform constants;

\item[(B2)] $|G_n| \gtrsim n^{3d/\min(\alpha,1)}$ and $\log|G_n| = O(\log n)$. 
\end{enumerate}

{
\begin{remark}
Although typically the choices of the grid points $G_n$ belong to the data analyst, in some applications the choices of design points are not completely free.
For example, in material synthesis experiments some environment parameter settings (e.g., temperature and pressure) might not be accessible due to budget or physical constraints.
Thus, we choose to consider less restrictive conditions imposed on the design grid $G_n$, allowing it to be more flexible in real-world applications.
\end{remark}
}

For any subset $S\subseteq G_n$ and a ``weight'' function $\varrho: G_n\to\mathbb R^+$, 
define the \emph{extension} $S^\circ(\varrho)$ of $S$ with respect to $\varrho$ as
\begin{equation}
S^\circ(\varrho) := \bigcup_{x\in S} B_{\varrho(x)}^\infty(x;G_n)\;\;\;\;\text{where}\;\;\;\;
B_{\varrho(x)}^\infty(x;G_n) = \{z\in G_n: \|z-x\|_\infty \leq \varrho(x)\}.
\label{eq:scirc}
\end{equation} 
The algorithm can then be formulated as two level of iterations, with the outer loop shrinking the ``active set'' $S_{\tau}$
and the inner loop collecting data that reduce lengths of confidence intervals on the active set.
A pseudocode description of our proposed algorithm is given in Fig.~\ref{fig:main-alg}.



\begin{algorithm}[t]
\SetAlgoLined
\SetAlgorithmName{Figure}{}

\KwParameters{$\alpha$, $M$, $\delta$, $n$}
\KwOutput{$\widehat x_n=x_n$, the final prediction}

Initialization: $S_0=G_n$, $\varrho_0(x)\equiv\infty$, $T = \lfloor \log_2 n\rfloor$, $n_0=\lfloor n/T\rfloor$\;

\For{$\tau=1,2,\ldots, T$}{
	Compute ``extended'' sample set $S_{\tau-1}^\circ(\varrho_{\tau-1})$ defined in Eq.~(\ref{eq:scirc})\;
	\For{$t=(\tau-1)n_0+1$ to $\tau n_0$}{
		Sample $x_t$ uniformly at random from $S_{\tau-1}^\circ(\varrho_{\tau-1})$ and observe $y_t=f(x_t)+w_t$\;
	}
	For every $x\in S_{\tau-1}$, find bandwidth $h_t(x)$ and build CI $[\ell_t(x),u_t(x)]$ in Eq.~(\ref{eq:ci})\;
	$S_\tau := \{x\in S_{\tau-1}: \ell_t(x)\leq\min_{x'\in S_{\tau-1}}u_t(x')\}$, $\varrho_\tau(x) := \min\{\varrho_{\tau-1}(x), h_t(x)\}$.
}

\vskip 0.1in
\caption{The main algorithm.}
\label{fig:main-alg}
\end{algorithm}

\subsection{Local Polynomial Regression}
We use local polynomial regression \citep{fan1996local} to obtain the estimate $\widehat f(x)$.
In particular, for any $x\in G_n$ and a bandwidth parameter $h>0$, consider a least square polynomial estimate
\begin{equation}
\widehat f_{h} \in \arg\min_{g\in\mathcal P_k} \sum_{t'=1}^t \mathbb I[x_{t'}\in B_{h}^{\infty}(x)]\cdot \left(y_{t'}-g(x_{t'})\right)^2,
\label{eq:lpr}
\end{equation}
where $B_{h}^\infty(x) := \{x'\in\mathcal X: \|x'-x\|_\infty\leq h\}$ and $\mathcal P_k$ denotes all polynomials of degree $k$ on $\mathcal X$.

To analyze the performance of $\widehat f_h$ evaluated at a certain point $x\in\mathcal X$,
define mapping $\psi_{x,h}: z\mapsto (1, \psi_{x,h}^1(z), \ldots,\psi_{x,h}^k(z))$ where $\psi_{x,h}^j: z\mapsto [\prod_{\ell=1}^j{h^{-1}(z_{i_\ell}-x_{i_\ell})}]_{i_1,\ldots,i_j=1}^d$ is the degree-$j$ polynomial mapping from $\mathbb R^d$ to $\mathbb R^{d^j}$.
Also define $\Psi_{t,h}:= (\psi_{x,h}(x_{t'}))_{1\leq t'\leq t, x_{t'}\in B_{h}(x)}$ as the $m\times D$ aggregated design matrix,
where $m=\sum_{t'=1}^t{\mathbb I[x_{t'}\in B_{h}^{\infty}(x)]}$ and $D=1+d+\ldots+d^k$, $k=\lfloor\alpha\rfloor$.
The estimate $\widehat f_h$ defined in Eq.~(\ref{eq:lpr}) then admits the following closed-form expression:
\begin{equation}
\widehat f_h(z) \equiv \psi_{x,h}(z)^\top (\Psi_{t,h}^\top\Psi_{t,h})^\dagger\Psi_{t,h}^\top Y_{t,h},
\label{eq:ols}
\end{equation}
where $Y_{t,h}=(y_{t'})_{1\leq t'\leq t, x_{t'}\in B_h^\infty(x)}$ and $A^\dagger$ is the Moore-Penrose pseudo-inverse of $A$.

The following lemma gives a finite-sample analysis of the error of $\hat f_h(x)$:
\begin{lemma}
Suppose $f$ satisfies Eq.~(\ref{eq:holder}) on $B_h^\infty(x;\mathcal X)$, $\max_{z\in B_h^\infty(x;\mathcal X)}\|\psi_{x,h}(z)\|_2\leq b$ and $\frac{1}{m}\Psi_{t,h}^\top\Psi_{t,h}\succeq \sigma I_{D\times D}$ for some $\sigma>0$.
Then for any $\delta\in(0,1/2)$, with probability $1-\delta$
\begin{equation}
\big|\widehat f_{h}(x) - f(x)\big| \leq \underbrace{\frac{b^2}{\sigma}Md^k h^\alpha}_{\sb_{h,\delta}(x)} + \underbrace{b\sqrt{\frac{5D\ln(1/\delta)}{\sigma m}}}_{\sv_{h,\delta}(x)} =: \eta_{h,\delta}(x).
\label{eq:lpr-error}
\end{equation}
\label{lem:lpr}
\end{lemma}
\begin{remark}
$\sb_{h,\delta}(x)$, $\sv_{h,\delta}(x)$ and $\eta_{h,\delta}(x)$ depend on $x$ becauses $\sigma$ depends on $\Psi_{t,h}$, which further depends on the sample points 
in the neighborhood $B_h^\infty(x;\mathcal X)$ of $x$.
\end{remark}

In the rest of the paper we define $\sb_{h,\delta}(x) := (b^2/\sigma) Md^kh^\alpha$ and $\sv_{h,\delta}(x):=b\sqrt{5D\ln(1/\delta)/\sigma m}$ as the bias and standard deviation terms
in the error of $\widehat f_h(x)$, respectively.
We also denote $\eta_{h,\delta}(x) := \sb_{h,\delta}(x) + \sv_{h,\delta}(x)$ as the overall error in $\widehat f_h(x)$.

Notice that when bandwidth $h$ increases, the bias term $\sb_{h,\delta}(x)$ is likely to increase too because of the $h^\alpha$ term;
on the other hand, with $h$ increasing the local neighborhood $B_h^{\infty}(x;\mathcal X)$ enlarges and would potentially contain more samples,
implying a larger $m$ and smaller standard deviation term $\sv_{h,\delta}(x)$.
A careful selection of bandwidth $h$ balances $\sb_{h,\delta}(x)$ and $\sv_{h,\delta}(x)$ and yields appropriate confidence intervals on $f(x)$,
a topic that is addressed in the next section.

\subsection{Bandwidth Selection and Confidence Intervals}

Given the expressions of bias $\sb_{h,\delta}(x)$ and standard deviation $\sv_{h,\delta}(x)$ in Eq.~(\ref{eq:lpr-error}),
the bandwidth $h_t(x)>0$ at epoch $t$ and point $x$ is selected as
\begin{equation}
 h_t(x) := \frac{j_t(x)}{n^2} \;\;\text{where}\;\; j_t(x):= \arg\max\left\{j\in\mathbb N, j\leq n^2: \sb_{j/n^2,\delta}(x) \leq \sv_{j/n^2,\delta}(x)\right\}.
\label{eq:ht}
\end{equation}
More specifically, $h_t(x)$ is the largest positive value in an evenly spaced grid $\{j/n^2\}$
such that the bias of $\widehat f_h(x)$ is smaller than its standard deviation.
Such bandwidth selection is in principle similar to the Lepski's method \citep{lepski1997optimal},
with the exception that an upper bound on the bias for any bandwidth parameter is known
and does not need to be estimated from data.


With the selection of bandwidth $h_t(x)$ at epoch $t$ and query point $x$, a confidence interval on $f(x)$ is constructed as
\begin{equation}
\ell_t(x) := \max_{1\leq t'\leq t} \left\{\widehat f_{h_{t'}(x)}(x)-\eta_{h_{t'}(x),\delta}(x)\right\} \;\;\text{and}\;\;
u_t(x) := \min_{1\leq t'\leq t} \left\{\widehat f_{h_{t'}(x)}(x) + \eta_{h_{t'}(x),\delta}(x)\right\}.
\label{eq:ci}
\end{equation}
Note that for any $x\in\mathcal X$, the lower confidence edge $\ell_t(x)$ is a non-decreasing function in $t$ and the upper confidence edge $u_t(x)$ is a non-increasing function in $t$.

\subsection{Pre-screening}\label{subsec:relax}

We describe a pre-screening procedure that relaxes the smoothness condition from $\kappa=\infty$ to $\kappa=\Omega(1)$,
meaning that only local smoothness of $f$ around its minimum values is required.
Let $n_0=\lfloor n/\log n\rfloor$, $x_1,\ldots,x_{n_0}$ be points i.i.d.~uniformly sampled from $\mathcal X$ and $y_1,\ldots,y_{n_0}$ be their corresponding responses.
For every grid point $x\in G_n$, perform the following:
\begin{enumerate}
\item Compute $\check f(x)$ as the average of all $y_i$ such that $\|x_i-x\|_\infty \leq n_0^{-1/2d}\log^3 n =: h_0$;
\item Remove all $x\in G_n$ from $S_0$ if $\check f(x) \geq \min_{z\in G_n}\check f(z) +1/\log n$.
\end{enumerate}
\begin{remark}
The $1/\log n$ term in removal condition $\check f(x)\geq\min_{z\in G_n}\check f(z)+1/\log n$ is not important,
and can be replaced with any sequence $\{\omega_n\}$ such that $\lim_{n\to\infty}\omega_n = 0$ and $\lim_{n\to\infty}\omega_n n^{t} = \infty$ for any $t>0$.
The readers are referred to the proof of Proposition \ref{prop:screening} in the appendix for the motivation of this term as well as the selection of the pre-screening bandwidth $h_0$.
\end{remark}

At a high level, the pre-screening step computes local averages of $y$ and remove grid points in $S_0=G_n$ whose estimated values are larger than the minimum in $G_n$.

%
%
To analyze the pre-screening step, we state the following proposition:
\begin{proposition}
Assume $f\in\Sigma_\kappa^\alpha(M)$ and
let $S_0'$ be the screened grid after step 2 of the pre-screening procedure. Then for sufficiently large $n$, with probability $1-O(n^{-1})$ we have
\begin{equation}
\min_{x\in S_0'}f(x) = \min_{z\in G_n}f(x) \;\;\;\;\text{and}\;\;\;\; S_0'\subseteq \bigcup_{x\in L_f(\kappa/2)}B_{h_0}^\infty(x;\mathcal X),
\end{equation}
where $L_f(\kappa/2) = \{x\in\mathcal X: f(x)\leq f^*+\kappa/2\}$.
\label{prop:screening}
\end{proposition}

{
To interpret Proposition \ref{prop:screening}, note that for sufficiently large $n$, $f\in\Sigma_\kappa^\alpha(M)$ implies $f$ being $\alpha$-H\"{o}lder smooth (i.e., $f$ satisfies Eq.~(\ref{eq:holder}))
on $\bigcup_{x\in L_f(\kappa/2)}B_{h_0}^\infty(x;\mathcal X)$, because $\kappa>0$ is a constant and $h_0\to 0$ as $n\to\infty$.
Subsequently, the proposition shows that with high probability, the pre-screening step will remove all grid points in $G_n$ in non-smooth regions of $f$,
while maintaining the global optimal solution.
This justifies the pre-processing step for $f\in\Sigma_\kappa^\alpha(M)$, because $f$ is smooth on the grid after pre-processing.

The proof of Proposition \ref{prop:screening} uses the fact that the local mean estimation is large provided that all data points in the local mean estimator are large,
regardless of their underlying smoothness.
The complete proof of Proposition \ref{prop:screening} is deferred to the appendix.
}

\section{Proofs of main theorems}

\subsection{Proof of Lemma \ref{lem:lpr}}

Our proof closely follows the analysis of asymptotic convergence rates for series estimators in the seminal work of \cite{newey1997convergence}.
We further work out all constants in the error bounds to arrive at a completely finite-sample result, which is then used to construct finite-sample confidence intervals.

We start with as polynomial interpolation results for all H\"{o}lder smooth functions in $B_{h_t}^\infty(x;\mathcal X)$.
\begin{lemma}
Suppose $f$ satisfies Eq.~(\ref{eq:holder}) on $B_h^\infty(x;\mathcal X)$. Then there exists $\widetilde f_x\in\mathcal P_k$ such that
\begin{equation}
\sup_{z\in B_{h}^\infty(x;\mathcal X)} \big|f(z)-\widetilde f_x(z)\big| \leq Md^kh^\alpha.
\end{equation}
\label{lem:poly-interpolation}
\end{lemma} 
\begin{proof}
Consider
\begin{equation}
\widetilde f_x(z) := f(x) + \sum_{j=1}^k\sum_{\alpha_1+\ldots+\alpha_d=j} \frac{\partial^j f(x)}{\partial x_1^{\alpha_1}\ldots\partial x_d^{\alpha_d}} \prod_{\ell=1}^d{(z_\ell-x_\ell)^{\alpha_\ell}}.
\label{eq:tilde-f}
\end{equation}
By Taylor expansion with Lagrangian remainders, there exists $\xi\in(0,1)$ such that 
\begin{equation}
\big|\widetilde f_x(z)-f(z)\big| \leq \sum_{\alpha_1+\ldots+\alpha_d=k} \big|f^{(\vct\alpha)}(x+\xi(z-x))-f^{(\vct\alpha)}(x)\big|\cdot \prod_{\ell=1}^d{|z_\ell-x_\ell|^{\alpha_\ell}}.
\end{equation}
Because $f$ satisfies Eq.~(\ref{eq:holder}) on $B_h^\infty(x;\mathcal X)$, we have that $|f^{(\vct\alpha)}(x+\xi(z-x))-f^{(\vct\alpha)}(x)| \leq M\cdot\|z-x\|_\infty^{\alpha-k}$.
Also note that $|z_\ell-x_\ell| \leq \|z-x\|_\infty \leq h$ for all $z\in B_{h}^\infty(x;\mathcal X)$. The lemma is thus proved.
\end{proof}

Using Eq.~(\ref{eq:ols}), the local polynomial estimate $\widehat f_h$ can be written as $\widehat f_h(z) \equiv \psi_{x,h}(z)^\top\widehat\theta_h$, where
\begin{equation}
\widehat\theta_h = (\Psi_{t,h}^\top\Psi_{t,h})^{-1}\Psi_{t,h}^\top Y_{t,h}.
\label{eq:ols2}
\end{equation}
In addition, because $\widetilde f_x\in\mathcal P_k$, there exists $\widetilde\theta\in\mathbb R^D$ such that $\widetilde f_x(z) \equiv \psi_{x,h}(z)^\top\widetilde\theta$.
Denote also that $F_{t,h} := (f(x_{t'}))_{1\leq t'\leq t, x_{t'}\in B_{h}^\infty(x)}$, $\Delta_{t,h} := (f(x_{t'}) - \widetilde f_x(x_{t'}))_{1\leq t'\leq t,x_{t'}\in B_{h}^\infty(x)}$
and $W_{t,h}:= (w_{t'})_{1\leq t'\leq t,x_{t'}\in B_{h}^\infty(x)}$.
Eq.~(\ref{eq:ols2}) can then be re-formulated as
\begin{align}
\widehat\theta_h &= (\Psi_{t,h}^\top\Psi_{t,h})^{-1}\Psi_{t,h}^\top\left[\Psi_{t,h}\widetilde\theta + \Delta_{t,h} + W_{t,h}\right]\\
&= \widetilde\theta + \left[\frac{1}{m}\Psi_{t,h}^\top\Psi_{t,h}\right]^{-1} \left[\frac{1}{m}\Psi_{t,h}^\top(\Delta_{t,h}+W_{t,h})\right].
\end{align}
Because $\frac{1}{m}\Psi_{t,h}^\top\Psi_{t,h} \succeq \sigma I_{D\times D}$ and $\sup_{z\in B_{h}^\infty(x)}\|\psi_{x,h}(z)\|_2\leq b$, we have that
\begin{equation}
\|\widehat\theta_h-\widetilde\theta\|_2 \leq \frac{b}{\sigma}\|\Delta_{t,h}\|_{\infty} + \left\|\left[\frac{1}{m}\Psi_{t,h}^\top\Psi_{t,h}\right]^{-1}\frac{1}{m}\Psi_{t,h}^\top W_t\right\|_2.
\end{equation}

Invoking Lemma \ref{lem:poly-interpolation} we have $\|\Delta_{t,h}\|_{\infty} \leq Md^kh^\alpha$.
In addition, because $W_t\sim\mathcal N_{m}(0, I_{m\times n})$, we have that
\begin{equation}
\left[\frac{1}{m}\Psi_{t,h}^\top\Psi_{t,h}\right]^{-1}\frac{1}{m}\Psi_{t,h}^\top W_t \sim \mathcal N_D\left(0, \frac{1}{m}\left[\frac{1}{m}\Psi_{t,h}^\top\Psi_{t,h}\right]^{-1}\right).
\end{equation}
Applying concentration inequalities for quadratic forms of Gaussian random vectors (Lemma \ref{lem:chisquare}),
with probability $1-\delta$ it holds that
\begin{equation}
\left\|\left[\frac{1}{m}\Psi_{t,h}^\top\Psi_{t,h}\right]^{-1}\frac{1}{m}\Psi_{t,h}^\top W_t\right\|_2 \leq \sqrt{\frac{5D\log(1/\delta)}{\sigma m}}.
\end{equation}
We then have that with probability $1-\delta$ that
\begin{equation}
\|\widehat\theta_h-\widetilde\theta\|_2 \leq \frac{b}{\sigma_h}Md^k h_t^\alpha + \sqrt{\frac{5D\log(1/\delta)}{\sigma m}}.
\end{equation}
Finally, noting that
\begin{align}
|\widehat f_h(x)-f(x)| = |\widehat f_h(x)-\widetilde f_x(x)| = |\psi(x)^\top(\widehat\theta_h-\widetilde\theta)| \leq b\|\widehat\theta_h -\widetilde\theta\|_2
\end{align}
we complete the proof of Lemma \ref{lem:lpr}.

%

\subsection{Proof of Theorem \ref{thm:upper}}\label{subsec:proof-upper}

In this section we prove Theorem \ref{thm:upper}.
We prove the theorem by considering every reference function $f_0\in\Sigma_\kappa^\alpha(M)\cap\Theta_{\mat C}$ separately.
For simplicity, we assume $\kappa=\infty$ throughout the proof.
The $0<\kappa<\infty$ can be handled by replacing $\mathcal X$ with $S_0$ which is the grid after the pre-screening step described in Section~\ref{subsec:relax}.
We also suppress dependency on $d,\alpha,M,\mat C,\underline p_0,\overline p_0$ in $O(\cdot)$, $\Omega(\cdot)$, $\Theta(\cdot)$, $\gtrsim$, $\lesssim$ and $\asymp$ notations.
We further suppress logarithmic terms of $n$ in $\widetilde O(\cdot)$ and $\widetilde\Omega(\cdot)$ notations.


The following lemma is our main lemma, which shows that the active set $S_\tau$ in our proposed algorithm shrinks geometrically 
before it reaches a certain level. 
To simplify notations, denote $\tilde c_0:=10c_0$ and (A2) then hold for all $\epsilon,\delta\in[0,\tilde c_0]$ for all $f_0\in\Theta_{\mat C}$.

\begin{lemma}
For $\tau=1,\ldots,T$ define $\varepsilon_\tau := \max\{\tilde c_0\cdot 2^{-\tau}, C_3[\varepsilon_n^\u(f_0)+n^{-1/2}]\log^2 n\}$,
where $C_3>0$ is a constant depending only on $d,\alpha,M,\underline p_0,\overline p_0$ and $\mat C$. Then for sufficiently large $n$,
with probability $1-O(n^{-1})$ the following holds uniformly for all outer iterations $\tau=1,\ldots,T$:
\begin{equation}
S_{\tau} \subseteq L_f(\varepsilon_\tau).
\end{equation}
\label{lem:geometric}
\end{lemma}

{
Lemma \ref{lem:geometric} shows that the level $\varepsilon_\tau$ in $L_f(\varepsilon_\tau)$ that contains $S_{\tau-1}$ shrinks \emph{geometriclly}, 
until the condition $\varepsilon_\tau\geq C_3[\varepsilon^\u_n(f_0)+n^{-1/2}]\log^2 n$ is violated.
If the condition is never violated, then at the end of the last epoch $\tau^*$ we have $\varepsilon_{\tau^*}=O(n^{-1})$ because $\tau^*=\log n$, in which case Theorem \ref{thm:upper} clearly holds.
On the other hand, because $S_{\tau}\subseteq S_{\tau-1}$ always holds, we have $\varepsilon_{\tau^*} \lesssim [\varepsilon_n^\u(f_0)+n^{-1/2}]\log^2 n$
which justifies the convergence rate in Theorem \ref{thm:upper}.
}

In the rest of this section we prove Lemma \ref{lem:geometric}.
We need several technical lemmas and propositions.
Except for Proposition \ref{prop:truthful} that is straightforward, the proofs of the other technical lemmas are deferred to the end of this section.

We first show that the grid $G_n$ is sufficiently dense for approximate optimization purposes.
Define $x_n^* :=\argmin_{x\in G_n}f(x)$ and $f_n^* := f(x_n^*)$.
We have the following lemma:
\begin{lemma}
Suppose (B1) and (B2) hold. Then with probability $1-O(n^{-1})$ the following holds:
\begin{enumerate}
\item $\sup_{x\in\mathcal X}\min_{x'\in G_n}\|x-x'\|_\infty = \widetilde O(n^{-3/\min(\alpha,1)})$;
\item $f_n^*-f^* = \widetilde O(n^{-3})$.
\end{enumerate}
\label{lem:grid-dense}
\end{lemma}

The next proposition shows that with high probability, the confidence intervals constructed in the algorithm are truthful and the successive rejection procedure
will never exclude the true optimizer of $f$ on $G_n$.
\begin{proposition}
Suppose $\delta = 1/n^4|G_n|$. Then with probability $1-O(n^{-1})$ the following holds:
\begin{enumerate}
\item $f(x)\in [\ell_t(x), u_t(x)]$ for all $1\leq t\leq n$ and $x\in G_n$;
\item $x_n^*\in S_{\tau}$ for all $0\leq \tau\leq n$.
\end{enumerate}
\label{prop:truthful}
\end{proposition}
\begin{proof}
The first property is true by applying the union bound over all $t=1,\ldots,n$ and $x\in  G_n$.
The second property then follows, because $\ell_t(x_n^*)\leq f_n^*$ and $\min_{x\in S_{\tau-1}}u_t(x) \geq f_n^*$ for all $\tau$.
\end{proof}

The following lemma shows that every small box centered around a certain sample point $x\in G_n$ contains a sufficient number of sample points
whose least eigenvalue can be bounded with high probability under the polynomial mapping $\psi_{x,h}$ defined in Section~\ref{subsec:upper}.
\begin{lemma}
For any $x\in G_n$, $1\leq m\leq n$ and $h>0$, let $K_{h,m}^1(x), \ldots, K_{h,m}^n(x)$ be $n$ independent point sets, where each point set consists of $m$ points sampled i.i.d.~uniformly at random
from $B_h^\infty(x;G_n)=G_n\cap B_{h}^\infty(x)$.
With probability $1-O(n^{-1})$ the following holds true uniformly for all $x\in G_n$, $h\in \{j/n^2: j\in\mathbb N, j\leq n^2\}$ and $K_{h,m}^{\ell}(x)$, $\ell\in[n]$ as $n\to\infty$:
\begin{enumerate}
\item $\sup_{h>0}\sup_{z\in B_h^\infty(x)}\|\psi_{x,h}(z)\|_2\asymp \Theta(1)$;
\item $|B_h^\infty(x;G_n)| \asymp h^d|G_n|$;
\item $\sigma_{\min}(K_{h,m}^\ell(x))\asymp \Theta(1)$ for all $m\geq\Omega(\log^2 n)$ and $m\leq |G_n|$, 
where $\sigma_{\min}(K_{h,m}^{\ell}(x))$ is the least eigenvalue of $\frac{1}{m}\sum_{z\in K_{h,m}^\ell(x)}\psi_{x,h}(z)\psi_{x,h}(z)^\top$.
\end{enumerate}
\label{lem:design}
\end{lemma}

{
\begin{remark}
It is possible to improve the concentration result in Eq.~(\ref{eq:bh-concentration}) using the strategies adopted in \citep{chaudhuri2014consistent}
based on sharper Bernstein type concentration inequalities.
Such improvements are, however, not important in establishing the main results of this paper.
\end{remark}
}

The next lemma shows that, the bandwidth $h_t$ selected at the end of each outer iteration $\tau$ is near-optimal, being sandwiched
between two quantities determined by the size of the active sample grid $\widetilde S_{\tau-1}:=S_{\tau-1}^\circ(\varrho_{\tau-1})$.
\begin{lemma}
There exist constants $C_1,C_2>0$ depending only on $d,\alpha, M, \underline p_0,\overline p_0$ and $\mat C$ such that 
with probability $1-O(n^{-1})$, the following holds for every outer iteration $\tau\in\{1, \ldots, T\}$ and all $x\in S_{\tau-1}$:
\begin{equation}
C_1 [\widetilde\nu_{\tau-1}n_0]^{-1/(2\alpha+d)} -\tau/n \leq \varrho_\tau(x) \leq h_t(x)\leq C_2 [\widetilde\nu_{\tau-1}n_0]^{-1/(2\alpha+d)}\log n + \tau/n ,
\end{equation}
where $\widetilde\nu_{\tau-1} := |G_n|/|\widetilde S_{\tau-1}|$.
\label{lem:select-ht}
\end{lemma}

%


We are now ready to state the proof of Lemma \ref{lem:geometric}, which is based on an inductive argument over the epochts $\tau=1,\ldots,T$.
\begin{proof} 
We use induction to prove this lemma.
%
For the base case $\tau=1$, because $\|f-f_0\|_\infty\leq\varepsilon_n^\u(f_0)$ and $\varepsilon_n^\u(f_0)\to 0$ as $n\to\infty$,
it suffices to prove that $S_1\subseteq L_{f_0}(\tilde c_0/4)$ for sufficiently large $n$.
Because $\tilde S_0=S_0=G_n$, invoking Lemmas \ref{lem:select-ht} and \ref{lem:lpr} we have that $|u_t(x)-\ell_t(x)| = \tilde O(n^{-\alpha/(2\alpha+d)})$
for all $x\in G_n$ with high probability at the end of the first outer iteration $\tau=1$.
Therefore, for sufficiently large $n$ we conclude that $\sup_{x\in G_n}|u_t(x)-\ell_t(x)|\leq c_0/8$ and hence $S_1\subseteq L_{f_0}(\tilde c_0/4)$.

We now prove the lemma for $\tau\geq 2$, assuming it holds for $\tau-1$.
We also assume that $n$ (and hence $n_0$) is sufficiently large,
such that the maximum CI length $\max_{x\in G}|u_t(x)-\ell_t(x)|$ after the first outer iteration $\tau=1$ is smaller than $c_0$,
where $c_0$ is a constant such that

{
Because $\|f-f_0\|_{\infty}\leq \varepsilon_n^\u(f_0)$ and $\varepsilon_{\tau-1}\geq C_3\varepsilon_n^\u(f_0)\log^2 n$,
for appropriately chosen constant $C_3$ that is not too small,
we have that $\|f-f_0\|_\infty\leq \varepsilon_{\tau-1}$.
By the inductive hypothesis we have
}
\begin{equation}
S_{\tau-1}\subseteq L_f(\varepsilon_{\tau-1}) \subseteq L_{f_0}(\varepsilon_{\tau-1} + \|f-f_0\|_\infty) \subseteq L_{f_0}(2\varepsilon_{\tau-1}).
\label{eq:mainlem-1}
\end{equation}
Subsequently, denoting $\rho_{\tau-1}^* := \max_{x\in S_{\tau-1}}\varrho_{\tau-1}(x)$ we have
\begin{equation}
\widetilde S_{\tau-1} = S_{\tau-1}^\circ \subseteq L_{f_0}^\circ(2\varepsilon_{\tau-1}, \rho_{\tau-1}^*).
\label{eq:mainlem-2}
\end{equation}

{
Let $\bigcup_{x\in H_n} B_{\rho_{\tau-1}^*}^2(x)$ be the smallest covering set of $L_{f_0}(2\varepsilon_{\tau-1})$,
meaning that $L_{f_0}(2\varepsilon_{\tau-1})\subseteq \bigcup_{x\in H_n} B_{\rho_{\tau-1}^*}^2(x)$,
where $B_{\rho_{\tau-1}^*}^2(x) = \{z\in\mathcal X: \|z-x\|_2\leq\rho_{\tau-1}^*\}$ is the $\ell_2$ ball of radius $\rho_{\tau-1}^*$ centered at $x$.
By (A2), we know that $|H_n|\lesssim 1 + [\rho_{\tau-1}^*]^{-d}\mu_{f_0}(2\varepsilon_{\tau-1})$.
In addition, the enlarged level set satisfies $L_{f_0}^\circ(2\varepsilon_{\tau-1},\rho_{\tau-1}^*) \subseteq \bigcup_{x\in H_n}B_{2\rho_{\tau-1}^*}^\infty(x)$.
Subsequently,
}
\begin{equation}
\mu_{f_0}^\circ(2\varepsilon_{\tau-1},\rho_{\tau-1}^*)
\lesssim |\mathcal H_n|\cdot [\rho_{\tau-1}^*]^d
 \lesssim \mu_{f_0}(2\varepsilon_{\tau-1}) + [\rho_{\tau-1}^*]^d.
\label{eq:mainlem-3}
\end{equation}
By Lemma \ref{lem:select-ht}, the monotonicity of $|\widetilde S_{\tau-1}|$ and the fact that $\underline p_0\leq p_X(z)\leq \overline p_0$ for all $z\in \mathcal X$, we have 
\begin{align}
\rho_{\tau-1}^* 
&\lesssim [\mu_{f}^\circ(\varepsilon_{\tau-1},\rho_{\tau-1}^*)]^{1/(2\alpha+d)} n_0^{-1/(2\alpha+d)}\log n\\
&\leq [\mu_{f_0}^\circ(2\varepsilon_{\tau-1},\rho_{\tau-1}^*)]^{1/(2\alpha+d)} n_0^{-1/(2\alpha+d)}\log n\\
&\lesssim  \left(\mu_{f_0}(2\varepsilon_{\tau-1}) + [\rho_{\tau-1}^*]^d\right)^{1/(2\alpha+d)} n_0^{-1/(2\alpha+d)}\log n.
\label{eq:mainlem-4}
\end{align}
Re-arranging terms on both sides of Eq.~(\ref{eq:mainlem-4}) we have
\begin{equation}
\rho_{\tau-1}^* \lesssim \max\left\{[\mu_{f_0}(2\varepsilon_{\tau-1})]^{\frac{1}{2\alpha+d}}n_0^{-\frac{1}{2\alpha+d}}\log n,\; n_0^{-\frac{1}{2\alpha}}\log n \right\}.
\label{eq:mainlem-5}
\end{equation}

On the other hand, according to the selection procedure of the bandwidth $h_t(x)$, we have that $\eta_{h_t(x),\delta}(x) \lesssim \sb_{h_t(x),\delta}(x)$.
Invoking Lemma \ref{lem:select-ht} we have for all $x\in S_{\tau-1}$ that
\begin{align}
\eta_{h_t(x),\delta}(x)
&\lesssim \sb_{h_t(x),\delta}(x)\lesssim [h_t(x)]^\alpha\\
&\lesssim [\widetilde\nu_{\tau-1}n_0]^{-\alpha/(2\alpha+d)}\log n\label{eq:mainlem-6}\\
&\lesssim [\widetilde\nu_{\tau-2}n_0]^{-\alpha/(2\alpha+d)}\log n\label{eq:mainlem-7}\\
&\lesssim [\rho_{\tau-1}^*]^\alpha\log n.\label{eq:mainlem-8}
\end{align}
Here Eq.~(\ref{eq:mainlem-6}) holds by invoking the upper bound on $h_t(x)$ in Lemma \ref{lem:select-ht},
Eq.~(\ref{eq:mainlem-7}) holds because $\widetilde\nu_{\tau-1}\geq\widetilde\nu_{\tau-2}$,
and Eq.~(\ref{eq:mainlem-8}) holds by again invoking the lower bound on $\varrho_{\tau-1}(x)$ in Lemma \ref{lem:select-ht}.
Combining Eqs.~(\ref{eq:mainlem-5},\ref{eq:mainlem-8}) we have
\begin{align}
\max_{x\in S_{\tau-1}} \eta_{h_t(x),\delta}(x)
& \lesssim \max\left\{[\mu_{f_0}(2\varepsilon_{\tau-1})]^{\frac{\alpha}{2\alpha+d}} n_0^{-\frac{\alpha}{2\alpha+d}}\log^2 n, \; n_0^{-\frac{1}{2}}\log n\right\}.
\label{eq:mainlem-9}
\end{align}

{
Recall that $n_0 = n/\log n$
and $\varepsilon^\u_n(f_0)\leq \varepsilon_{\tau-1}$, provided that $C_3$ is not too small.
By definition, 
every $\varepsilon\geq\varepsilon_n^\u(f_0)$ satisfies $\varepsilon^{-(2+d/\alpha)}\mu_{f_0}(\varepsilon)\leq n/\log^\omega n$
for some large constant $\omega > 5+d/\alpha$.
Subsequently, 
\begin{align}
[\mu_{f_0}(2\varepsilon_{\tau-1})]^{\frac{\alpha}{2\alpha+d}} n_0^{-\frac{\alpha}{2\alpha+d}}\log^2 
&\lesssim 2\varepsilon_{\tau-1}n^{\frac{\alpha}{2\alpha+d}} \log^{-\frac{\omega\alpha}{2\alpha+d}} n\cdot n_0^{-\frac{\alpha}{2\alpha+d}}\log^2 n\\
&\lesssim \varepsilon_{\tau-1}/[\log n]^{\frac{(\omega-5-d/\alpha)\alpha}{2\alpha+d}}.
\label{eq:mainlem-last}
\end{align}
Because $\omega > 5+d/\alpha$, the right-hand side of Eq.~(\ref{eq:mainlem-last}) is asymptotically dominated
\footnote{We say $\{a_n\}$ is asymptotically dominated by $\{b_n\}$ if $\lim_{n\to\infty} |a_n|/|b_n| = 0$.}
 by $\varepsilon_{\tau-1}$.
In addition, $n_0^{-1/2}\log n$ is also asymptotically dominated by $\varepsilon_{\tau-1}$ because $\varepsilon_{\tau-1}\geq C_3n^{-1/2}\log^\omega n$.
Therefore, for sufficiently large $n$ we have
\begin{equation}
\max_{x\in S_{\tau-1}} \eta_{h_t(x),\delta}(x) \leq \varepsilon_{\tau-1}/4.
\end{equation}
Lemma \ref{lem:geometric} is thus proved.
}
\end{proof}

\subsubsection{Proof of Lemma \ref{lem:grid-dense}}\label{subsec:proof-grid-dense}
\begin{proof}
Let $H_N\subseteq\mathcal X$ be the finite subset of $\mathcal X$ such that $|H_N|=N$ and $\sup_{x\in\mathcal X}\min_{x'\in H_n}\|x-x'\|_{\infty}$ is maximized.
By standard results of metric entropy number of the $d$-dimensional unit box (see for example, \citep[Lemma 2.2]{van2000empirical}), we have that $\sup_{x\in\mathcal X}\min_{x'\in H_n}\|x-x'\|_{\infty} \lesssim N^{-1/d}$.

For any $x\in H_n$, consider an $\ell_\infty$ ball $B_{r_n}^\infty(x)$ or radius $r_n$ centered at $x$, with $r_n$ to be specified later.
Because the density of $P_X$ is uniformly bounded away from below on $\mathcal X$, we have that $P_X(x\in B_{r_n}^\infty(x)) \gtrsim r_n^d$.
Therefore, applying union bound over all $x\in H_n$ we have that 
\begin{align}
P_X\left[\exists x\in H_N, G_n\cap B_{r_n}^\infty(x)=\emptyset\right]
&\leq N(1-r_n^d)^{|G_n|} \lesssim \exp\left\{-r_n^d|G_n| + \log N\right\}.
\end{align}
Set $N=|G_n|$ and $r_n \asymp n^{-3/\min(\alpha,1)}\log n$.
The right-hand side of the above inequality is then upper bounded by $O(1/n^2)$, thanks to the assumption (A1) and that $|G_n|\gtrsim n^{3d/\min(\alpha,1)}$.
The first property is then proved by noting that
\begin{equation}
\sup_{x\in\mathcal X}\min_{x'\in G_n}\|x-x'\|_\infty 
\leq \sup_{x\in\mathcal X}\min_{x'\in H_n}\|x-x'\|_{\infty} + \max_{x\in H_n}\min_{x'\in G_n}\|x-x'\|_\infty.
\end{equation}

To prove the second property, note that for any $x,x'\in\mathcal X$, $|f(x)-f(x')|\leq M\cdot \|x-x'\|_{\infty}^{\min(\alpha,1)}$.
The first property then implies that $f_n^*-f^*=\widetilde O(n^{-3})$.
\end{proof}

\subsubsection{Proof of Lemma \ref{lem:design}}\label{subsec:proof-lemma-design}
\begin{proof}
We first show that the first property holds almost surely.
Recall the definition of $\psi_{x,h}$, we have that $1\leq \|\psi_{x,h}(z)\|_2 \leq D\cdot [\max_{1\leq j\leq d}h^{-1}|z_j-x_j|]^k$. 
Because $\|z-x\|_\infty\leq h$ for all $z\in B_h^\infty(x)$, $\sup_{z\in B_h^\infty(x)}\|\psi_{x,h}(z)\|_2\lesssim O(1)$ for all $h>0$.
Thus, $\sup_{h>0}\sup_{z\in B_h^\infty(x)}\|\psi_{x,h}(z)\|_2\asymp\Theta(1)$ for all $x\in G_n$.

For the second property, by Hoeffding's inequality (Lemma \ref{lem:hoeffding}) and the union bound, with probability $1-O(n^{-1})$ we have that
\begin{equation}
\max_{x, h}\left|\frac{ |B_h^\infty(x;G_n)|}{|G_n|} - P_X(z\in B_h^\infty(x))\right| \lesssim \sqrt{\frac{\log n}{|G_n|}}.
\label{eq:bh-concentration}
\end{equation}
In addition, note that $P_X(z\in B_h^\infty(x;\mathcal X)) \geq \underline p_0\lambda(B_h^\infty(x;\mathcal X)) \gtrsim h^d$
and $P_X(z\in B_h^\infty(x;\mathcal X)) \leq \overline p_0\lambda(B_h^\infty(x;\mathcal X))\lesssim h^d$,
where $\lambda(\cdot)$ denotes the Lebesgue measure on $\mathcal X$.
Subsequently, $|B_h^\infty(x;G_n)|$ is lower bounded by $\Omega(h^d|G_n| - \sqrt{|G_n|\log n})$
and upper bounded by $O(h^d|G_n| + \sqrt{|G_n|\log n})$.
The second property is then proved by noting that $h_d\gtrsim n^{-d}$ and $|G_n|\gtrsim n^{3d/\min(\alpha,1)}$.

We next prove the third property.
 Because $\underline p_0\leq p_X(z)\in \overline p_0$ for all $z\in\mathcal X$, we have that
\begin{align}
\underline p_0 \int_{B_h^\infty(x;\mathcal X)} \psi_{x,h}(z)\psi_{x,h}(z)^\top \ud U_{x,h}(z) 
&\preceq \mathbb E\left[\frac{1}{m}\sum_{z\in K_{h,m}^\ell}\psi_{x,h}(z)\psi_{x,h}(z)^\top\right]\\
&\preceq \overline p_0 \int_{B_h^\infty(x;\mathcal X)} \psi_{x,h}(z)\psi_{x,h}(z)^\top \ud U_{x,h}(z),
\end{align}
where $U_{x,h}$ is the uniform distribution on $B_h^\infty(x;\mathcal X)$.
Note also that 
\begin{align}
\int_{\mathcal X}\psi_{0,1}(z)\psi_{0,1}(z)^\top\ud U(z)
&\preceq\int_{B_h^\infty(x;\mathcal X)} \psi_{x,h}(z)\psi_{x,h}(z)^\top \ud U_{x,h}(z) \label{eq:lowerbound-std}\\
&\preceq 2^d\int_{\mathcal X}\psi_{0,1}(z)\psi_{0,1}(z)^\top\ud U(z)\label{eq:upperbound-std}
\end{align}
where $U$ is the uniform distribution on $\mathcal X=[0,1]^d$.
The following proposition upper and lower bounds the eigenvalues of $\int_{\mathcal X}\psi_{0,1}(z)\psi_{0,1}(z)^\top\ud U(z)$,
which is proved in the appendix.
\begin{proposition}
There exist constants $0<\psi_0\leq \Psi_0<\infty$ depending only on $d,D$ such that 
\begin{equation}
\psi_0 I_{D\times D}\preceq \int_{\mathcal X}\psi_{0,1}(z)\psi_{0,1}(z)^\top\ud U(z)\preceq \Psi_0 I_{D\times D}.
\label{eq:nonsingular-std}
\end{equation}
\label{prop:nonsingular}
\end{proposition}

Using Proposition \ref{prop:nonsingular} and Eqs.~(\ref{eq:lowerbound-std},\ref{eq:upperbound-std}), we conclude that
\begin{equation}
\Omega(1)\cdot I_{D\times D}\preceq \mathbb E\left[\frac{1}{m}\sum_{z\in K_{h,m}^\ell}\psi_{x,h}(z)\psi_{x,h}(z)^\top\right]
\preceq O(1)\cdot I_{D\times D}.
\label{eq:econdition}
\end{equation}
Applying matrix Chernoff bound (Lemma \ref{lem:matrix-chernoff}) and the union bound, we have that with probability $1-O(n^{-1})$,
\begin{equation}
\max_{x,h,m,\ell}\left\|\frac{1}{m}\sum_{z\in K_{h,m}^\ell(x)}\psi_{x,h}(z)\psi_{x,h}(z)^\top - \mathbb E\left[\psi_{x,h}(z)\psi_{x,h}(z)^\top|z\in B_h(x)\right]\right\|_\op \lesssim \sqrt{\frac{\log n}{m}}.
\label{eq:perturbcondition}
\end{equation} 
Combining Eqs.~(\ref{eq:econdition},\ref{eq:perturbcondition}) and applying Weyl's inequality (Lemma \ref{lem:weyl}) we have 
\begin{equation}
\Omega(1) - O(\sqrt{\log n/m}) \lesssim \sigma_{\min}(K_{h,m}^\ell(x)) \lesssim O(1) - O(\sqrt{\log n/m}) .
\end{equation}
The third property is therefore proved.
\end{proof}

\subsubsection{Proof of Lemma \ref{lem:select-ht}}\label{subsec:proof-select-ht}
\begin{proof}
We use induction to prove this lemma.
For the base case of $\tau=1$, we have $\widetilde S_{0} = S_0=G_n$ and therefore $\widetilde\nu_{\tau-1}=1$.
Furthermore, applying Lemma \ref{lem:design} we have that for all $h=j/n^2$, 
\begin{equation}
\sb_{h,\delta}(x) \asymp h^\alpha \;\;\;\;\;\text{and}\;\;\;\;\;
\sv_{h,\delta}(x) \asymp \sqrt{\frac{\log n}{h^d n_0}}.
\end{equation}
Thus, for $h$ selected according to Eq.~(\ref{eq:ht}) as the largest bandwidth of the form $j/n^2$, $j\in\mathbb N$ such that $\sb_{h,\delta}(x)\leq \sv_{h,\delta}(x)$,
both $\sb_{h,\delta}(x), \sv_{h,\delta}(x)$ are on the order of $n_0^{-1/(2\alpha+d)}$ up to logarithmic terms of $n$,
 and therefore
one can pick appropriate constants $C_1,C_2>0$ such that $C_1n_0^{-1/(2\alpha+d)}\leq \varrho_1(x)\leq C_2n_0^{-1/(2\alpha+d)} \log n$
holds for all $x\in G_n$.

We next prove the lemma for $\tau>1$, assuming it holds for $\tau-1$.
We first establish the lower bound part.
Define $\rho_{\tau-1}^* := \min_{z\in S_{\tau-1}}\varrho_{\tau-1}(z)$.
By inductive hypothesis, $\rho_{\tau-1}^*\geq C_1[\widetilde\nu_{\tau-2}n_0]^{-1/(2\alpha+d)}-(\tau-1)/n$.
Note also that $\widetilde\nu_{\tau-1}\geq\widetilde\nu_{\tau-2}$ because $\widetilde S_{\tau-1}\subseteq\widetilde S_{\tau-2}$,
which holds because $S_{\tau-1}\subseteq S_{\tau-2}$ and $\varrho_{\tau-1}(z)\leq\varrho_{\tau-2}(z)$ for all $z$.
Let $h_t^*$ be the smallest number of the form $j_t^*/n^2$, $j_t^*\in[n^2]$ such that $h_t^* \geq C_1[\widetilde\nu_{\tau-1}n_0]^{-1/(2\alpha+d)}-\tau/n$.
We then have $h_t^*\leq \rho_{\tau-1}^*$ and therefore query points in epoch $\tau$ are uniformly distributed in $B_{h_t^*}^{\infty}(x;G_n)$.
Subsequently, applying Lemma \ref{lem:design} we have with probability $1-O(n^{-1})$ that 
\begin{equation}
\sb_{h_t^*,\delta}(x) \leq C' [h_t^*]^\alpha \;\;\;\;\;\text{and}\;\;\;\;\;
\sv_{h_t^*,\delta}(x) \geq C''\sqrt{\frac{\log n}{[h_t^*]^{d}\widetilde\nu_{\tau-1}n}},
\end{equation}
where $C',C''>0$ are constants that depend on $d,\alpha,M,\underline p_0,\overline p_0$ and $\mat C$, but \emph{not} $C_1$, $C_2$, $\tau$ or $h_t^*$.
By choosing $C_1$ appropriately (depending on $C'$ and $C''$) we can make $\sb_{h_t^*,\delta}(x)\leq \sv_{h_t^*,\delta}(x)$ holds for all $x\in S_{\tau-1}$,
thus establishing $\varrho_{\tau}(x)\geq\min\{\varrho_{\tau-1}(x),h_t^*\} \geq C_1[\widetilde\nu_{\tau-1}n_0]^{-1/(2\alpha+d)}- \tau/n$.

We next prove the upper bound part.
For any $h_t=j_t/n^2$ where $j_t\in[n^2]$, invoking Lemma \ref{lem:design} we have that 
\begin{equation}
\sb_{h,\delta}(x) \geq \widetilde C'h^\alpha \;\;\;\;\;\text{and}\;\;\;\;\;
\sv_{h,\delta}(x) \leq \widetilde C''\sqrt{\frac{\log n}{\min\{h, \rho_{\tau-1}^*\}^d\cdot \widetilde\nu_{\tau-1}n_0}},
\end{equation}
where $\tilde C'$ and $\tilde C''$ are again constants depending on $d,\alpha,M,\underline p_0,\overline p_0$ and $\mat C$, but \emph{not} $C_1,C_2$.
Note also that $\rho_{\tau-1}^*\geq C_1[\widetilde\nu_{\tau-2}n_0]^{-1/(2\alpha+d)}-(\tau-1)/n\geq C_1[\widetilde\nu_{\tau-1}n_0]^{-1/(2\alpha+d)} -\tau/n$,
because $\tilde\nu_{\tau-1}\geq\tilde\nu_{\tau-2}$.
By selecting constant $C_2>0$ carefully (depending on $\widetilde C',\widetilde C''$ and $C_1$), we can ensure $\sb_{h,\delta}(x) > \sv_{h,\delta}(x)$ for all
$h\geq C_2[\widetilde\nu_{\tau-1}n_0]^{-1/(2\alpha+d)}+\tau/n$.
Therefore, $\varrho_\tau(x)\leq h_t(x) \leq C_2[\widetilde\nu_{\tau-1}n_0]^{-1/(2\alpha+d)}+\tau/n$.
\end{proof}

\subsection{Proof of Theorem \ref{thm:lower}}\label{subsec:proof-lower}
%

{
In this section we prove the main negative result in Theorem \ref{thm:lower}.
To simplify presentation, we suppress dependency on $\alpha,d,c_0$ and $C_0$ in $\lesssim, \gtrsim, \asymp, O(\cdot)$ and $\Omega(\cdot)$ notations.
However, we do \emph{not} suppress dependency on $\underline C_R$ or $M$ in any of the above notations.
}

Let $\varphi_0:[-2,2]^d\to\mathbb R^*$ be a non-negative function defined on $\mathcal X$ such that $\varphi_0\in \Sigma_\kappa^{\lceil\alpha\rceil}(1)$ with $\kappa=\infty$,
$\sup_{x\in\mathcal X}\varphi_0(x) =\Omega(1)$
and $\varphi_0(z)=0$ for all $\|z\|_2\geq 1$.
Here $\lceil\alpha\rceil$ denotes the smallest integer that upper bounds $\alpha$.
Such functions exist and are the cornerstones of the construction of information-theoretic lower bounds in nonparametric estimation problems \citep{castro2008minimax}.
One typical example is the {``smoothstep''} function (see for example \citep{ebert2003texturing})
$$
S_N(x) := \frac{1}{Z}x^{N+1}\sum_{n=0}^N\binom{N+n}{n}\binom{2N+1}{N-n}(-x)^n, \;\;\;\;\;\;N=0,1,2,\ldots
$$
where $Z>0$ is a scaling parameter.
The smoothstep function $S_N$ is defined on $[0,1]$ and satisfies the H\"{o}lder condition in Eq.~(\ref{eq:holder}) of order $\alpha=N$ on $[0,1]$.
It can be easily extended to $\tilde S_{N,d}:[-2,2]^d\to\mathbb R$ by considering
$\tilde S_{N,d}(x) := 1/Z-S_N(a\|x\|_1)$ where $\|x\|_1=|x_1|+\ldots+|x_d|$ and $a=1/(2d)$.
It is easy to verify that, with $Z$ chosen appropriately,
 $\tilde S_{N,d}\in\Sigma_{\infty}^N(1)$, $\sup_{x\in\mathcal X}\tilde S_{N,d}(x)=1/Z=\Omega(1)$ and $\tilde S_{N,d}(z)=0$ for all $\|z\|_2\geq 1$,
where $M>0$ is a constant.

For any $x\in\mathcal X$ and $h>0$, define $\varphi_{x,h}:\mathcal X\to\mathbb R^*$ as
\begin{equation}
\varphi_{x,h}(z) := \mathbb I[z\in B_h^\infty(x)]\cdot \frac{Mh^\alpha}{2}\varphi_0\left(\frac{z-x}{h}\right).
\label{eq:varphi}
\end{equation}
It is easy to verify that $\varphi_{x,h}\in \Sigma_\infty^\alpha(M/2)$, and furthermore $\sup_{z\in\mathcal X}\varphi_{x,h}(z) \asymp Mh^\alpha$ and
$\varphi_{x,h}(z)=0$ for all $z\notin B_h^\infty(x)$.


Let $L_{f_0}(\varepsilon_n^\l(f_0))$ be the level set of $f_0$ at $\varepsilon_n^\l(f_0)$.
%
Let $H_n\subseteq L_{f_0}(\varepsilon_n^\l(f_0))$ be the largest \emph{packing} set such that $B_h^\infty(x)$ are disjoint for all $x\in H_n$,
and $\bigcup_{x\in H_n}B_h^\infty(x)\subseteq L_{f_0}(\varepsilon_n^\l(f_0))$.
By (A2') and the definition of $\varepsilon_n^\l(f_0)$, we have that 
\begin{equation}
|H_n|\geq M(L_{f_0}(\varepsilon_n^\l(f_0)),2\sqrt{d}h) \gtrsim \mu_{f_0}(\varepsilon_n^\l(f_0)) \cdot h^{-d}
\geq[\varepsilon_n^\l(f_0)]^{2+d/\alpha}\cdot nh^{-d}.
\end{equation}


For any $x\in H_n$, construct $f_x:\mathcal X\to\mathbb R$ as
\begin{equation}
f_x(z) := f_0(z) - \varphi_{x,h}(z).
\label{eq:fx}
\end{equation}
Let $\mathcal F_n := \{f_x: x\in H_n\}$ be the class of functions indexed by $x\in H_n$.
Let also $h\asymp (\varepsilon_n^\l(f_0)/M)^{1/\alpha}$ such that $\|\varphi_{x,h}\|_{\infty} = 2\varepsilon^\l_n(f_0)$.
We then have that $\|f_x-f_0\|_\infty\leq 2\varepsilon_n^\l(f_0)$ and $f_x\in\Sigma_\infty^\alpha(M)$, because $f_0,\varphi_{x,h}\in\Sigma_\infty^\alpha(M/2)$.

The next lemma shows that, with $n$ adaptive queries to the noisy zeroth-order oracle $y_t=f(x_t)+w_t$,
it is information theoretically not possible to identify a certain $f_x$ in $\mathcal F_n$ with high probability.
\begin{lemma}
Suppose $|\mathcal F_n|\geq 2$.
Let $\mathcal A_n=(\chi_1,\ldots,\chi_n,\phi_n)$ be an active optimization algorithm operating with a sample budget $n$, 
which consists of \emph{samplers} $\chi_\ell: \{(x_i,y_i)\}_{i=1}^{\ell-1}\mapsto x_\ell$
and an \emph{estimator} $\phi_n: \{(x_i,y_i)\}_{i=1}^n\mapsto\widehat f_x\in\mathcal F_n$,
both can be deterministic or randomized functions.
Then
\begin{equation}
\inf_{\mathcal A_n}\sup_{f_x\in\mathcal F_n} \Pr_{f_x}\left[\widehat f_x\neq f_x\right] \geq \frac{1}{2} - \sqrt{\frac{n\cdot \sup_{f_x\in\mathcal F_n}\|f_x-f_0\|_\infty^2}{2|\mathcal F_n|}}.
\label{eq:lower-main}
\end{equation}
\label{lem:lower-main}
\end{lemma}

\begin{lemma}
There exists constant $M>0$ depending on $\alpha,d,c_0,C_0$ such that
the right-hand side of Eq.~(\ref{eq:lower-main}) is lower bounded by $1/3$.
\label{lem:lower-one-quarter}
\end{lemma}

Lemmas \ref{lem:lower-main} and \ref{lem:lower-one-quarter} are proved at the end of this section.
Combining both lemmas and noting that for any distinct $f_x,f_{x'}\in\mathcal F_n$ and $z\in\mathcal X$, $\max\{\sL(z;f_x),\sL(z;f_{x'})\}\geq \varepsilon_n^\l(f_0)$,
 we proved the minimax lower bound formulated in Theorem \ref{thm:lower}.
\subsubsection{Proof of Lemma \ref{lem:lower-main}}
Our proof is inspired by the negative result of multi-arm bandit pure exploration problems established in \citep{bubeck2009pure}.

\begin{proof} 
For any $x\in H_n$, define
\begin{equation}
n_x := \mathbb E_{f_0}\left[\sum_{i=1}^n{\mathbb I[x\in B_h^\infty(x)]}\right].
\end{equation}
Because $B_h^\infty(x)$ are disjoint for $x\in H_n$,
we have $\sum_{x\in H_n}n_x \leq n$.
Also define, for every $x\in H_n$,
\begin{equation}
\wp_x := \Pr_{f_0}\left[\widehat f_x= f_x\right].
\end{equation}
Because $\sum_{x\in H_n}\wp_x = 1$, by pigeonhole principle there is at most one $x\in H_n$ such that $\wp_x > 1/2$.
Let $x_1,x_2\in H_n$ be the points that have the largest and second largest $n_x$.
Then there exists $x\in\{x_1,x_2\}$ such that $\wp_x\leq 1/2$ and $n_x\leq 2n/|\mathcal F_n|$.
By Le Cam's and Pinsker's inequality (see, for example, \citep{tsybakov2009introduction}) we have that
\begin{align}
\Pr_{f_x}\left[\widehat f_x= f_x\right]
&\leq \Pr_{f_0}\left[\widehat f_x= f_x\right] + d_{\mathrm{TV}}(P_{f_0}^{\mathcal A_n}\|P_{f_x}^{\mathcal A_n})\\
&\leq \Pr_{f_0}\left[\widehat f_x= f_x\right] + \sqrt{\frac{1}{2}\kl(P_{f_0}^{\mathcal A_n}\|P_{f_x}^{\mathcal A_n})}\\
&= \wp_x + \sqrt{\frac{1}{2}\kl(P_{f_0}^{\mathcal A_n}\|P_{f_x}^{\mathcal A_n})}\\
&\leq \frac{1}{2} + \sqrt{\frac{1}{2}\kl(P_{f_0}^{\mathcal A_n}\|P_{f_x}^{\mathcal A_n})}.
\end{align}

It remains to upper bound KL divergence of the active queries made by $\mathcal A_n$.
Using the standard lower bound analysis for active learning algorithms \citep{castro2008minimax,castro2014adaptive} and the fact that 
$f_x\equiv f_0$ on $\mathcal X\backslash B_h^\infty(x)$, we have
\begin{align}
\kl(P_{f_0}^{\mathcal A_n}\|P_{f_x}^{\mathcal A_n})
&= \mathbb E_{f_0,\mathcal A_n}\left[\log\frac{P_{f_0,\mathcal A_n}(x_{1:n},y_{1:n})}{P_{f_x,\mathcal A_n}(x_{1:n},y_{1:n})}\right]\\
&= \mathbb E_{f_0,\mathcal A_n}\left[\log\frac{\prod_{i=1}^n{P_{f_0}(y_i|x_i)P_{\mathcal A_n}(x_i|x_{1:(i-1)},y_{1:(i-1)})}}{\prod_{i=1}^n{P_{f_x}(y_i|x_i)P_{\mathcal A_n}(x_i|x_{1:(i-1)},y_{1:(i-1)})}}\right]\\
&= \mathbb E_{f_0,\mathcal A_n}\left[\log\frac{\prod_{i=1}^nP_{f_0}(y_i|x_i)}{\prod_{i=1}^nP_{f_x}(y_i|x_i)}\right]\\
&= \mathbb E_{f_0,\mathcal A_n}\left[\sum_{x_i\in B_h(x)}\log\frac{P_{f_0}(y_i|x_i)}{P_{f_x}(y_i|x_i)}\right]\\
&\leq n_x\cdot \sup_{z\in B_h^\infty(x;\mathcal X)}\kl(P_{f_0}(\cdot|z)\|P_{f_x}(\cdot|z))\\
&\leq n_x\cdot \|f_0-f_x\|_\infty^2. 
\end{align}
Therefore,
\begin{equation}
\Pr_{f_x}\left[\widehat f_x= f_x\right] 
\leq \frac{1}{2} + \sqrt{\frac{1}{4}n_x\varepsilon_n^2} \leq \frac{1}{2} +\sqrt{\frac{n\|f_x-f_0\|_\infty^2}{2|\mathcal F_n|}}.
\end{equation}
\end{proof}

\subsubsection{Proof of Lemma \ref{lem:lower-one-quarter}}

\begin{proof}
By construction, $n\sup_{f_x\in\mathcal F_x}\|f_x-f_0\|_\infty^2 \lesssim M^2 nh^{2\alpha}$ and
$|\mathcal F_n|=|H_n|\gtrsim [\underline C_{\varepsilon}\varepsilon_n^\l(f_0)]^{2+d/\alpha}nh^{-d}$.
Note also that $h\asymp (\varepsilon/M)^{1/\alpha} \asymp (\underline C_{\varepsilon}\varepsilon_n^\l(f_0)/M)^{1/\alpha}$
because $\|f_x-f_0\|_\infty = \varepsilon = \underline C_{\varepsilon}\varepsilon_n^\l(f_0)$.
Subsequently,
\begin{align}
\frac{n\sup_{f_x\in\mathcal F_x}\|f_x-f_0\|_\infty^2}{2|\mathcal F_n|}
\lesssim \frac{n [\underline C_\varepsilon \varepsilon_n^\l(f_0)]^2}{n[\underline C_\varepsilon\varepsilon_n^\l(f_0)]^2\cdot M^{d/\alpha}}
= M^{-d/\alpha}.
\end{align}

By choosing the constant $M>0$ to be sufficiently large, 
the right-hand side of the above inequality is upper bounded by $1/36$.
The lemma is thus proved.
\end{proof}

\subsection{Proof of Theorem \ref{thm:lower-passive}}

The proof of Theorem \ref{thm:lower-passive} is similar to the proof of Theorem \ref{thm:lower},
but is much more standard by invoking the \emph{Fano's inequality} \citep{tsybakov2009introduction}.
In particular, adapting the Fano's inequality on any finite function class $\mathcal F_n$ constructed we have the following lemma:
\begin{lemma}[Fano's inequality]
Suppose $|\mathcal F_n|\geq 2$, and $\{(x_i,y_i)\}_{i=1}^n$ are i.i.d.~random variables.
Then
\begin{equation}
\inf_{\hat f_x}\sup_{f_x\in\mathcal F_n} \Pr_{f_x}\left[\hat f_x\neq f_x\right] \geq 1 - \frac{\log 2 + n\cdot \sup_{f_x,f_{x'}\in\mathcal F_n}\kl(P_{f_x}\|P_{f_{x'}})}{\log |\mathcal F_n|},
\label{eq:fano}
\end{equation}
where $P_{f_x}$ denotes the distribution of $(x,y)$ under the law of $f_x$.
\label{lem:fano}
\end{lemma}

{
Let $\mathcal F_n$ be the function class constructed in the previous proof of Theorem \ref{thm:lower},
corresponding to the largest packing set $H_n$ of $L_{f_0}(\tilde\varepsilon_n^\l)$
such that $B_h^\infty(x)$ for all $x\in H_n$ are disjoint,
where $h\asymp (\tilde\varepsilon_n^\l/M)^{1/\alpha}$ such that $\|\varphi_{x,h}\|_\infty = 2\tilde\varepsilon_n^\l$
for all $x\in H_n$.
Because $f_0$ satisfies (A2'), we have that $|\mathcal F_n|= |H_n| \gtrsim \mu_{f_0}(\tilde\varepsilon_n^\l)h^{-d}$.
Under the condition that $\varepsilon_n^\u(f_0)\leq \tilde\varepsilon_n^\l$, 
it holds that $\mu_{f_0}(\tilde\varepsilon_n^\l) \geq [\tilde\varepsilon_n^\l]^{2+d/\alpha} n$.
Therefore,
\begin{equation}
|\mathcal F_n|\gtrsim [\tilde\varepsilon_n^\l]^{2+d/\alpha}\cdot nh^{-d} \gtrsim [\tilde\varepsilon_n^\l]^2\cdot nM^{d/\alpha}.
\end{equation}
Because $\log(n/\tilde\varepsilon_n^\l)\gtrsim \log n$ and $M>0$ is a constant,
we have that $\log|\mathcal F_n| \geq c\log n$ for all $n\geq N$,
where $c>0$ is a constant depending only on $\alpha,d$ and $N\in\mathbb N$ is a constand depending on $M$.
}

Let $U$ be the uniform distribution on $\mathcal X$.
Because $x\sim U$ and $f_x\equiv f_{x'}$ on $\mathcal X\backslash B_h^\infty(x)$, we have that
\begin{align}
\kl(P_{f_x}\|P_{f_{x'}}) 
&= \frac{1}{2}\int_{\mathcal X}|f_x(z)-f_{x'}(z)|^2\ud U(z)\\
&\leq \frac{1}{2}\Pr_{U}\left[z\in B_h^\infty(x)\right]\cdot \|f_x-f_{x'}\|_{\infty}^2\\
&\leq \frac{1}{2}\lambda(B_h^\infty(x))\cdot [\varepsilon_n^\l]^2\\
&\lesssim h^d[\tilde\varepsilon_n^\l]^2 \lesssim [\tilde\varepsilon_n^\l]^{2+d/\alpha} / M^{d/\alpha}.
\end{align}
By choosing $M$ to be sufficiently large, the right-hand side of Eq.~(\ref{eq:fano}) can be lower bounded by an absolute constant.
The theorem is then proved following the same argument as in the proof of Theorem \ref{thm:lower}.

\section{Conclusion}

In this paper we consider the problem of noisy zeroth-order optimization of general smooth functions.
Matching lower and upper bounds on the local minimax convergence rates are established, which are significantly different from classical minimax rates in nonparametric regression problems.
Many interesting future directions exist along this line of research, including exploitation of additive structures in the underlying function $f$ to completely remove curse of dimensionality, 
 functions with spatially heterogeneous smoothness or level set growth behaviors,
 and to design more computationally efficient algorithms that work well in practice.

\appendix

\section{Some concentration inequalities}

\begin{lemma}[\citep{hoeffding1963probability}]
Suppose $X_1,\ldots,X_n$ are i.i.d.~random variables such that $a\leq X_i\leq b$ almost surely. Then for any $t>0$,
$$
\Pr\left[\left|\frac{1}{n}\sum_{i=1}^n{X_i}- \mathbb EX\right| > t\right] \leq 2\exp\left\{-\frac{nt^2}{2(b-a)^2}\right\}.
$$
\label{lem:hoeffding}
\end{lemma}

\begin{lemma}[\citep{hsu2012tail}]
Suppose $x\sim\mathcal N_d(0, I_{d\times d})$ and let $A$ be a $d\times d$ positive semi-definite matrix.
Then for all $t>0$, 
$$
\Pr\left[x^\top A x > \tr(A) + 2\sqrt{\tr(A^2)t} + 2\|A\|_\op t\right] \leq e^{-t}.
$$
\label{lem:chisquare}
\end{lemma}

\begin{lemma}[\citep{tropp2015introduction}, simplified]
Suppose $A_1,\ldots,A_n$ are i.i.d.~positive semidefinite random matrices of dimension $d$ and $\|A_i\|_\op\leq R$ almost surely.
Then for any $t>0$,
$$
\Pr\left[\left\|\frac{1}{n}\sum_{i=1}^n{A_i} - \mathbb EA\right\|_\op > t\right] \leq 2\exp\left\{-\frac{nt^2}{8R^2}\right\}.
$$
\label{lem:matrix-chernoff}
\end{lemma}

\begin{lemma}[Weyl's inequality]
Let $A$ and $A+E$ be $d\times d$ matrices with $\sigma_1,\ldots,\sigma_d$ and $\sigma_1',\ldots,\sigma_d'$ be their singular values, sorted in descending order.
Then $\max_{1\leq i\leq d}|\sigma_i-\sigma_i'|\leq \|E\|_\op$.
\label{lem:weyl}
\end{lemma}

\section{Additional proofs}

\begin{proof}[Proof of Proposition \ref{prop:reduction}]
Consider arbitrary $x^*\in\mathcal X$ such that $f(x^*)=\inf_{x\in\mathcal X}f(x)$. {}
Then we have that $\sL(\hat x_n;f) = f(\hat x_n)-f(x^*) \leq [\hat f_n(\hat x_n)+\|\hat f_n-f\|_\infty] - [\hat f_n(x^*)-\|\hat f_n-f\|_\infty] \leq 2\|\hat f_n-f\|_\infty$,
where the last inequality holds because $\hat f_n(\hat x_n)\leq \hat f_n(x^*)$ by optimality of $\hat x_n$.
\end{proof}

\begin{proof}[Proof of Example \ref{exmp:convex}]
Because $f_0\in\Sigma_\kappa^2(M)$ is strongly convex, there exists $\sigma>0$ such that $\nabla^2 f_0(x)\succeq \sigma I$ for all $x\in\mathcal X_{f_0,\kappa}$,
where $\mathcal X_{f_0,\kappa} := L_{f_0}(\kappa)$ is the $\kappa$-level set of $f_0$.
Let $x^*=\arg\min_{x\in\mathcal X}f_0(x)$, which is unique because $f_0$ is strongly convex.
The smoothness and strong convexity of $f_0$ implies that 
\begin{equation}
f_0^*+ \frac{\sigma}{2}\|x-x^*\|_\infty^2 \leq f_0(x) \leq f_0^* + \frac{M}{2}\|x-x^*\|_\infty^2 \;\;\;\;\;\;\forall x\in\mathcal X_{f_0,\kappa}.
\end{equation}
Subsequently, there exist constants $c_0,C_1,C_2>0$ depending only on $\sigma,M,\kappa$ and $d$ such that for all $\epsilon\in(0,c_0]$, 
\begin{equation}
B_{C_1\sqrt{\epsilon}}^\infty(x^*;\mathcal X) \subseteq L_{f_0}(\epsilon) \subseteq B_{C_2\sqrt{\epsilon}}^\infty(x^*;\mathcal X).
\end{equation}

The property $\mu_{f_0}(\epsilon)\lesssim \epsilon^\beta$ holds because $\mu(L_{f_0}(\epsilon)) \geq \mu(B_{C_1\sqrt{\epsilon}}^\infty(x^*;\mathcal X)) \gtrsim \epsilon^{d/2}$.
To prove (A2), note that $N(L_{f_0}(\epsilon),\delta) \leq N(B_{C_2\sqrt{\epsilon}}^\infty(x^*;\mathcal X), \delta) \lesssim 1 + (\sqrt{\epsilon}/\delta)^d$.
Because $\epsilon^{d/2}\lesssim \mu(L_{f_0}(\epsilon)) = \mu_{f_0}(\epsilon)$, we conclude that $N(L_{f_0}(\epsilon),\delta)\lesssim 1 + \delta^{-d}\mu_{f_0}(\epsilon)$
and (A2) is thus proved.
\end{proof}

\begin{proof}[Proof of Proposition \ref{prop:negative-examples}]
Consider $f_0\equiv 0$ if $\beta=0$ and $f_0(z) := a_0\left[z_1^p + \ldots + z_d^p\right]$ for all $z=(z_1,\ldots,z_d)\in[0,1]^d$, where $a_0>0$ is a constant depending on $\alpha,M$, and $p=d/\beta$ for $\beta\in(0,d/\alpha]$.
The $\beta=0$ case where $f_0\equiv 0$ trivially holds. So we shall only consider the case of $\beta\in(0,d/\alpha]$.

We first show $f_0\in\Sigma_\kappa^\alpha(M)$ with $\kappa=\infty$, provided that $a_0$ is sufficiently small.
For any $j\leq k=\lfloor\alpha\rfloor$ and $\alpha_1+\ldots+\alpha_d=j$, we have
\begin{equation}
\frac{\partial^j}{\partial x_1^{\alpha_1}\ldots\partial x_d^{\alpha_d}}f_0(z) = \left\{\begin{array}{ll}
a_0j!\cdot z_\ell^{p-j}& \text{if }\alpha_\ell=j, \ell\in[d];\\
0& \text{otherwise}.\end{array}\right.
\end{equation}
Because $z_1,\ldots,z_d\in[0,1]$ and $p=d/\beta\geq\alpha\geq j$, it's clear that $0\leq \partial^j f_0(z)/\partial x_1^{\alpha_1}\ldots \partial x_d^{\alpha_d} \leq a_0 j!$.
In addition, for any $z,z'\in[0,1]^d$ and $\alpha_\ell=k$, $\ell\in[d]$, we have
\begin{align}
\left|\frac{\partial^k}{\partial x_1^{\alpha_1}\ldots\partial x_d^{\alpha_d}}f_0(z) - \frac{\partial^k}{\partial x_1^{\alpha_1}\ldots\partial x_d^{\alpha_d}}f_0(z')\right|
&\leq a_0 k!  \cdot\big|[z_\ell]^{p-k}-[z_\ell']^{p-k}\big|\\
&\leq a_0 k! \cdot\big| z_\ell-z_\ell'\big|^{\min\{p-k,1\}},
\end{align}
where the last inequality holds because $x^t$ is $\min\{t,1\}$-H\"{o}lder continuous on $[0,1]$ for $t\geq 0$.
The $|z_\ell-z_\ell'|^{\min\{p-k,1\}}$ term can be further upper bounded by $\|z-z'\|_\infty^{\alpha-k}$, because $p=d/\beta\geq \alpha$.
By selecting $a_0>0$ to be sufficiently small (depending on $M$) we have $f_0\in\Sigma_{\infty}^\alpha(M)$.

We next prove $f_0$ satisfies $\mu_{f_0}(\epsilon)\asymp \epsilon^\beta$ with parameter $\beta$ depending on $a_0$ and $p$.
For any $\epsilon>0$, the level set $L_{f_0}(\epsilon)$ can be expressed as
$L_{f_0}(\epsilon) = \{z\in[0,1]^d: z_1^p+\ldots+z_d^p\leq \epsilon/a_0\}$.
Subsequently, 
\begin{equation}
\left[0, \left(\frac{\epsilon}{a_0 d}\right)^{1/p}\right]^d\subseteq L_{f_0}(\epsilon) \subseteq \left[0, \left(\frac{\epsilon}{a_0}\right)^{1/p}\right]^d.
\end{equation}
Therefore,
\begin{equation}
[\epsilon/(a_0d)]^{dp} \leq \mu_{f_0}(\epsilon) \leq [\epsilon/a_0]^{dp}.
\end{equation}
Because $a_0,d$ are constants and $dp=\beta$, 
we established $\mu_{f_0}(\epsilon)\asymp \epsilon^\beta$ for $\beta=dp$. 

Finally, note that for any $\epsilon>0$, $L_{f_0}(\epsilon)$ is sandwiched between two cubics whose volumes only differ by a constant.
This proves (A2) and (A2') on the covering and packing numbers of $L_{f_0}(\epsilon)$.
\end{proof}

\begin{proof}[Proof of Proposition \ref{prop:screening}]
By Chernoff bound and union bound, with probability $1-O(n^{-1})$ uniformly over all $x\in G_n$, there are $\Omega(\sqrt{n_0}\log^2 n)$ uniform samples in $B_{h_0}^{\infty}(x;\mathcal X)$.
Subsequently, by standard Gaussian concentration inequality, with probability $1-O(n^{-1})$ we have
\begin{equation}
\inf_{z\in B_{h_0}^\infty(x;\mathcal X)}f(z) - O(n_0^{-1/4}) \leq \check f(x) \leq \sup_{z\in B_{h_0}^\infty(x;\mathcal X)}f(z) + O(n_0^{-1/4})\;\;\;\;\;\;\forall x\in G_n.
\end{equation}

Fix arbitrary $\tilde x^*\in\arg\min_{x\in G_n}f(x)$. 
Because $f\in\Sigma_\kappa^\alpha(M)$ for constant $\kappa$ and $h_0\to 0$,
$f$ is smooth on $B_{h_0}^\infty(\tilde x^*;\mathcal X)$ and therefore $\sup_{z\in B_{h_0}^{\infty}(\tilde x^*;\mathcal X)}f(z) \leq f(\tilde x^*) + O(h_0^{\min\{\alpha,1\}})
\leq f(\tilde x^*) + O(1/\log^2 n) \leq f^* + O(1/\log^2 n)$,
where the last inequality holds due to Lemma \ref{lem:grid-dense}.
On the other hand, for all $x\in G_n$, 
$\check f(x) \geq f^* - O(n_0^{-1/4})$.
Therefore, for sufficiently large $n$ we must have $\check f(\tilde x^*) \leq \min_{z\in G_n}\check f(z) + 1/\log n$ and subsequently $\tilde x^*\in S_0'$.

We next prove the statement that $S_0'\subseteq \bigcup_{x\in L_f(\kappa/2)}B_{h_0}^{\infty}(x;\mathcal X)$.
Consider arbitrary $z\in G_n$ and $z\notin\bigcup_{x\in L_f(\kappa/2)}B_{h_0}^{\infty}(x;\mathcal X)$.
By definition, $f(z')\geq f^*+\kappa/2$ for all $z'\in B_{h_0}^{\infty}(z;\mathcal X)$.
Subsequently, $\check f(z) \geq f^*+\kappa/2-O(n_0^{-1/4}) > f^* + 1/\log n$ for constant $\kappa>0$ and sufficiently large $n$, which implies $z\notin S_0'$.
\end{proof}

\begin{proof}[Proof of Proposition \ref{prop:nonsingular}]
The upper bound part of Eq.~(\ref{eq:nonsingular-std}) trivially holds because the absolute values of every element in $\psi_{0,1}(z)\psi_{0,1}(z)^\top$
for $z\in\mathcal X=[0,1]^d$ is upper bounded by $O(1)$.
To prove the lower bound part, we only need to show $\int_{\mathcal X}\psi_{0,1}(z)\psi_{0,1}(z)^\top\ud U(z)$ is invertible.
Assume the contrary.
Then there exists $v\in\mathbb R^D\backslash\{0\}$ such that 
\begin{equation}
v^\top\left[\int_{\mathcal X}\psi_{0,1}(z)\psi_{0,1}(z)^\top\ud U(z)\right]v
= \int_{\mathcal X}\big|\psi_{0,1}(z)^\top v\big|^2\ud U(z) = 0.
\end{equation}

Therefore, $\langle\psi_{0,1}(z),v\rangle = 0$ almost everywhere on $z\in[0,1]^d$.
Because $h>0$, by re-scaling with constants this implies the existence of non-zero coefficient vector $\xi$ such that 
\begin{equation*}
P(z_1,\ldots,z_m) := \sum_{\alpha_1+\ldots+\alpha_m\leq k} \xi_{\alpha_1,\ldots,\alpha_m}z_1^{\alpha_1}\ldots z_m^{\alpha_m} = 0 \;\;\text{almost everywhere on $z\in[0,1]^d$}.
\end{equation*}

We next use induction to show that, for any degree-$k$ polynomial $P$ of $s$ variables $z_1,\ldots,z_s$ that has at least one non-zero coefficient, 
the set $\{z_1,\ldots,z_s\in [0,1]^d: P(z_1,\ldots,z_s)=0\}$ must have zero measure.
This would then result in the desired contradiction.
For the base case of $s=1$, the fundamental theorem of algebra asserts that $P(z_1)=0$ can have at most $k$ roots,
which is a finite set and of measure 0.

We next consider the case where $P(z_1,\ldots,z_s)$ takes on $s$ variables. Re-organizing the terms we have 
\begin{equation}
P(z_1,\ldots,z_s) \equiv P_0(z_1,\ldots,z_{s-1}) + z_s P_1(z_1,\ldots,z_{s-1}) + \ldots + z_s^k P_k(z_1,\ldots,z_{s-1}),
\end{equation}
where $P_1,\ldots,P_k$ are degree-$k$ polynomials of $z_1,\ldots,z_{s-1}$.
Because $P$ has a non-zero coefficient, at least one $P_j$ must also have a non-zero coefficient.
By the inductive hypothesis, the set $\{z_1,\ldots,z_{s-1}: P_j(z_1,\ldots,z_{s-1})\}$ has measure 0.
On the other hand, if $P_j(z_1,\ldots,z_{s-1})\neq 0$, then invoking the fundamental theorem of algebra again on $z_s$
we know that there are finitely many $z_s$ such that $P(z_1,\ldots,z_s)=0$.
Therefore, $\{z_1,\ldots,z_s: P(z_1,\ldots,z_s)=0\}$ must also have measure zero.
\end{proof}

\section*{Acknowledgments}
The work of SB was supported in part by the NSF grant DMS-17130003. 

\bibliography{refs}
\bibliographystyle{plain}

\end{document}